\def\year{2022}\relax
\documentclass[letterpaper]{article} 
\usepackage{aaai22}  
\usepackage{times}  
\usepackage{helvet}  
\usepackage{courier}  
\usepackage[hyphens]{url}  
\usepackage{graphicx} 
\usepackage{natbib}  
\usepackage{caption} 
\DeclareCaptionStyle{ruled}{labelfont=normalfont,labelsep=colon,strut=off} 
\usepackage{algorithm}
\usepackage{algorithmic}

\usepackage{booktabs}
\usepackage{newfloat}
\usepackage{listings}
\floatstyle{ruled}
\newfloat{listing}{tb}{lst}{}
\floatname{listing}{Listing}
\title{Stability Verification in Stochastic Control Systems via Neural Network Supermartingales}

\author {
    Mathias Lechner\thanks{Equal Contribution.},
    \DJ{}or\dj{}e \v{Z}ikeli\'c\footnotemark[1],
    Krishnendu Chatterjee,
    Thomas A. Henzinger \\\
}
\affiliations {
    IST Austria\\\
    Klosterneuburg, Austria\\\
    \{mlechner, dzikelic,krishnendu.chatterjee,tah\}@ist.ac.at
}

\usepackage{amssymb}
\usepackage{amsthm}
\usepackage{amsmath}
\usepackage[T1]{fontenc}
\usepackage[utf8]{inputenc}
\usepackage{xcolor}
\usepackage{multirow}

\usepackage[symbol]{footmisc}

\newtheorem{definition}{Definition}
\newtheorem{theorem}{Theorem}

\newtheorem{proposition}{Proposition}

\newtheorem{manualtheoreminner}{\bf Theorem}
\newenvironment{manualtheorem}[1]{%
	\renewcommand\themanualtheoreminner{#1}%
	\manualtheoreminner
}{\endmanualtheoreminner}

\newcommand{\eps}{\epsilon}

\newcommand{\support}{\mathsf{support}}
\newcommand{\Stable}{\mathcal{X}_s}

\begin{document}

\maketitle

\begin{abstract}
We consider the problem of formally verifying almost-sure (a.s.) asymptotic stability in discrete-time nonlinear stochastic control systems. While verifying stability in deterministic control systems is extensively studied in the literature, verifying stability in stochastic control systems is an open problem. The few existing works on this topic either consider only specialized forms of stochasticity or make restrictive assumptions on the system, rendering them inapplicable to learning algorithms with neural network policies. 
In this work, we present an approach for general nonlinear stochastic control problems with two novel aspects: (a)~instead of classical stochastic extensions of Lyapunov functions, we use ranking supermartingales (RSMs) to certify a.s.~asymptotic stability, and (b)~we present a method for learning neural network RSMs. 
We prove that our approach guarantees a.s.~asymptotic stability of the system and
provides the first method to obtain bounds on the stabilization time, which stochastic Lyapunov functions do not.
Finally, we validate our approach experimentally on a set of nonlinear stochastic reinforcement learning environments with neural network policies.
\end{abstract}

\section{Introduction}

Reinforcement learning (RL) presents a promising approach to learning high-performing control policies in nonlinear control problems. However, most RL algorithms focus on learning a policy that maximizes the expected reward~\cite{sutton2018reinforcement}, and do not take safety constraints into account. This raises concerns about their suitability for safety-critical applications such as autonomous vehicles or healthcare. Thus, a fundamental challenge for the deployment of policies learned via RL algorithms in safety-critical applications is certifying their safety~\cite{AmodeiOSCSM16}.

Stability is one of the most important properties that a control policy needs to ensure for the system to be safe~\cite{lyapunov1992general}. In their training phase, RL algorithms explore unknown environments through randomized actions while optimizing the learned policy's expected reward. Without a control mechanism to ensure that the system safely recovers from such exploratory actions and goes back to some known safe region, this might lead to catastrophic events. For example~\cite{BerkenkampTS017}, if a self-driving car ends up driving outside its lane, a safe control policy needs to be able to stabilize the car back within the lane. Stability analysis is concerned with providing formal guarantees that the system can, with probability $1$, recover back to this safe region from any system state and stay there indefinitely.

Formal verification of stability in {\em deterministic} control problems is well-studied. In particular, Lyapunov functions are an established method for stability analysis of deterministic systems~\cite{khalil2002nonlinear}. A more recent line of research focuses on automatically learning a Lyapunov function in the form of a neural network (see Related Work).
While there are several theoretical results on extending Lyapunov functions to stochastic systems~\cite{kushner1965stability,Kushner14}, only a few works consider the problem of automated stability analysis for control problems where the stochasticity originates from the environment~\cite{CrespoS03,Vaidya15}. Moreover, these works rely on restrictive assumptions on the system, making them inapplicable to learning algorithms with neural network policies, and they only verify strictly weaker notions of stability. Since uncertainty is a crucial component of RL systems, through exploration and to bridge the simulation-to-real gap \cite{tobin2017domain,james2017transferring}, methods for stability verification in stochastic control problems are needed. These methods need to support neural network policies and truly certify stability.

In this work, we present a method for formally verifying stability in discrete-time stochastic control problems. Our method is based on {\em ranking supermartingales (RSMs)}, which were originally introduced in the programming languages literature for termination analysis of probabilistic programs~\cite{ChakarovS13}. Intuitively, RSMs are 
nonnegative functions that decrease in expectation by at least $\eps>0$ after every one-step evolution of the system and in each state that is not in the target region. We prove, for the first time, that RSMs can also be used to define stability certificates for stochastic control problems.

There are two key advantages of using RSMs instead of existing stochastic extensions of Lyapunov functions. First, we show that the defining properties of RSMs are much easier to encode within a learning framework. Second, we show that RSMs provide the first method to obtain bounds on the {\em stabilization time}, which stochastic Lyapunov functions do not. Ensuring that stabilization happens within some tolerable time limit is another practical concern about system safety. For instance, given a stabilizing policy for a self-driving car that drives at a very high speed, it is not sufficient to {\em only} ensure that the speed eventually stabilizes within the allowed speed limit. A good stabilizing policy in such scenarios {\em additionally} needs to provide plausible guarantees on the stabilization time. One of the key benefits of using RSMs is that they provide such guarantees.

We then proceed to presenting an algorithmic framework for learning RSMs in the form of neural networks. Our algorithm draws insight from existing methods for learning Lyapunov functions in deterministic control problems~\cite{ChangRG19}, and consists of two modules: the {\em learner} which learns an RSM candidate in the form of a neural network, and the {\em verifier} which then verifies the learned candidate. Whenever the verification step fails, a set of counterexamples showing that the candidate is not an RSM is computed, which are then used by the learner to fine-tune the candidate. This loop is repeated until a learned RSM candidate is successfully verified.

One of the key algorithmic challenges in designing the verifier module, compared to the case of deterministic systems, is that we need to verify an expected decrease condition which requires being able to compute the {\em expected value} of a neural network function over a probability distribution. Note, sampling cannot be used for this task since it only allows computing statistical bounds.
To solve this challenge, we propose a method for efficiently computing formal and tight bounds on the expected value of an arbitrary neural network function over a probability distribution. We also demonstrate experimentally that our method computes tight bounds in practice.
Our algorithmic contribution on computing expected value bounds for neural networks might on its own open various research directions on analyzing neural networks in probabilistic settings.

Finally, we evaluate our approach on two stochastic RL tasks with neural network control policies. It successfully learns RSMs proving that the policies stabilize the systems.  

\paragraph{Contributions} Our contributions are as follows:
\begin{enumerate}
    \item We show that ranking supermartingales (RSMs) provide a stability certificate for stochastic control problems, as well as guarantees on the stabilization time.
    
    \item We present a framework for learning neural network RSMs which also formally verifies the learned RSM.
    
    \item As a part of our verification framework, we present a method for efficiently computing formal bounds on the expected value of a neural network function over a probability distribution. We are not aware of any existing works that tackle this problem.
    
    \item We empirically validate that our approach can prove stability of stochastic systems with neural network policies.
\end{enumerate}

\section{Related Work}

\paragraph{Stability verification via Lyapunov functions} Stability verification in {\em deterministic} dynamical systems has received a lot of attention in recent works. For systems with polynomial dynamics and Lyapunov functions restricted to the sum-of-squares (SOS) form, a Lyapunov function can be computed via semi-definite programming~\cite{henrion2005positive,parrilo2000structured,jarvis2003some}. A learner-verifier framework similar to ours but for computing polynomial Lyapunov functions has been proposed in~\cite{RavanbakhshS19}. However, these methods require polynomial approximations and may not be efficient for systems with general nonlinearities. Moreover, it is known that even some simple dynamical systems that are asymptotically stable do not admit polynomial Lyapunov functions~\cite{AhmadiKP11}.

Learning Lyapunov functions in the form of a neural network has been considered in~\cite{RichardsB018,ChangRG19,AbateAGP21}, and it is an approach that is better suited to dynamical systems with general nonlinearities. In particular, \cite{RichardsB018} learn a Lyapunov function together with a region in which the system is stable by first discretizing the state space of the system, then learning a Lyapunov function candidate which tries to maximize the number of the discrete states at which the Lyapunov condition holds, and finally verifying that the candidate is indeed a Lyapunov function. The works~\cite{ChangRG19,AbateAGP21} propose a learner-verifier framework which uses counterexamples found by the verifier to improve the loss function and thus learn a new candidate. This loop is repeated until the verifier certifies that the Lyapunov function is correct. Our method for stability verification combines and extends ideas from these works.

\paragraph{Stability for stochastic control problems} All of the above methods consider deterministic dynamical systems. While there are several theoretical results on the stability of stochastic dynamical systems (see~\cite{Kushner14} for a comprehensive survey), to our best knowledge there are very few works that consider their automated stability verification~\cite{Vaidya15,CrespoS03}. Both of these are numerical approaches that first partition the system's state space into finitely many regions and then over-approximate the system's continuous dynamics via a discrete finite-state abstraction. Thus, the computed stability certificates are piecewise-constant. Furthermore, \cite{Vaidya15} verifies a weaker notion of stability called ``coarse stochastic stability'' that depends on the partition of the state space, and \cite{CrespoS03} imposes stability by requiring the system to reach the stabilizing region within some pre-specified finite time and deterministically (i.e.~for each sample path).


\paragraph{Reachability for deterministic control problems}
Reachability is a property that is naturally related to stability since stability requires reachability of the stabilization set. There are several approaches and tools that analyze reachability in {\em deterministic} continuous-time feedback loop systems controlled by neural network policies. Some notable examples are Sherlock~\cite{DuttaCS19} and ReachNN/ReachNN$^\ast$~\cite{HuangFLC019,FanHCL020} which use polynomial approximations to efficiently over-approximate the reachable set over some given time horizon, NNV~\cite{TranYLMNXBJ20} which is based on abstract interpretation, LRT-NG \cite{GruenbacherCLIS20} which overapproximates the reachable set as sequence of hyperspheres, or Verisig~\cite{IvanovWAPL19} which reduces the problem to reachability analysis in hybrid systems. Furthermore, GoTube \cite{gruenbacher2021gotube} constructs the reachable set of a deterministic continuous-time system with statistical guarantees about the constructed set overapproximating the true reachable states.

Note, however, that the goal of reachability analysis is to compute a set of states that are visited by {\em some} trajecotry of the system. In contrast, the goal of stability analysis is to show that {\em all} trajectories stabilize within the stabilization set (or with probability $1$ in the case of stochastic systems). Furthermore, the above tools consider reachability over some finite time horizon and in deterministic systems, whereas in this work we do not impose any time limit and consider stochastic systems. Thus, these tools are not applicable to the stability verification problem in stochastic control systems.

\paragraph{Safe exploration} RL algorithms need to explore the environment via randomized actions to learn which actions lead to a high future reward. However, in safety-critical environments, random actions may lead to catastrophic results. Safe exploration RL aims to restrict the exploratory actions to those that ensure safety of the environment. The most dominant approach to addressing this problem is learning the system dynamics' uncertainty bounds and limiting the exploratory actions within a high probability safety region. In the literature, Gaussian Processes \cite{Koller2018LearningBasedMP, Turchetta2019SafeEF, Berkenkamp2019SafeEI}, linearized models \cite{Dalal2018SafeEI} , deep robust regression \cite{Liu2020RobustRF}, and Bayesian neural networks \cite{lechner2021infinite} are used for learning the uncertainty bounds.

\paragraph{Learning stable dynamics} Learning dynamics from observation data is the first step in many control methods as well as model-based RL. Recent works considered learning deterministic system dynamics with guarantees on stability of some specified region~\cite{KolterM19}. Learning stochastic dynamics from observation data has been studied in~\cite{UmlauftH17,LawrenceLFBG20}. 

\paragraph{RSMs for probabilistic programs} Ranking supermartingales (RSMs) were first introduced in the programming languages community in order to reason about termination of probabilistic programs (PPs)~\cite{ChakarovS13}. They are a stochastic extension of the classical notion of ranking functions in programs~\cite{Floyd1967}, and in addition to ensuring probability $1$ termination they also provide guarantees on termination time~\cite{FioritiH15,ChatterjeeFNH16}. Our theoretical guarantees on the stabilization time draw insight from these results.

While some of our theoretical results are motivated by the works on PPs, our approach to stability verification differs significantly from the existing methods for RSM computation in PPs~\cite{ChakarovS13,ChatterjeeFNH16,ChatterjeeFG16}. In particular, these methods compute linear/polynomial RSMs via linear/semi-definite programming, and are more similar to the early methods for the computation of  polynomial Lyapunov functions that we discussed above. On the contrary, our method learns an RSM in the form of a neural network. The only method for learning RSMs in PPs has been presented in the recent work of~\cite{AbateGR20}, but this work computes only neural network RSMs with a single hidden layer and for a restricted class of PPs. In contrast, one of the main algorithmic novelties of our work is that we propose a general framework for computing the expected value of a neural network function over some probability distribution, which allows us to learn multi-layer neural network RSMs for general nonlinear systems.

\section{Preliminaries}

We consider a discrete-time stochastic dynamical system
\[ \mathbf{x}_{t+1} = f(\mathbf{x}_t, \mathbf{u}_t, \mathbf{\omega}_t), \, t\in\mathbb{N}_0. \]
The dynamics of the system are defined by the dynamics function $f:\mathcal{X}\times\mathcal{U} \times \mathcal{N}\rightarrow \mathcal{X}$, where $\mathcal{X}\subseteq\mathbb{R}^m$ is the state space, $\mathcal{U}\subseteq\mathbb{R}^n$ is the control action space and $\mathcal{N}\subseteq \mathbb{R}^p$ is the disturbance space. The system starts in some initial state $\mathbf{x}_0\in \mathcal{X}$ and at each time step $t$, given a state $\mathbf{x}_t$, the action $\mathbf{u}_t = \pi(\mathbf{x}_t)$ is chosen according to a control policy $\pi:\mathcal{X}\rightarrow\mathcal{U}$. The action $\mathbf{u}_t$, the state $\mathbf{x}_t$ and a randomly sampled disturbance vector $\omega_t\sim d$ then give rise to the subsequent state $\mathbf{x}_{t+1}=f(\mathbf{x}_t,\mathbf{u}_t,\omega_t)$. Here, we use $d$ to denote the probability distribution over $\mathcal{N}$ from which the disturbance vector is sampled. Thus, the dynamics function $f$, the policy $\pi$ and the probability distribution $d$ together form a stochastic feedback loop system (or a closed-loop system).

\paragraph{Model assumptions} Stability analysis of stochastic dynamical systems would be impossible without additional assumptions on the system, so that the model is sufficiently well-behaved. To that end, we assume that $\mathcal{X}$, $\mathcal{U}$ and $\mathcal{N}$ are all Borel-measurable for the system semantics to be well-defined, and that $\mathcal{X}$ is compact in the Euclidean topology of $\mathbb{R}^m$. The dynamics function $f$ and the control policy $\pi$ are assumed to be Lipschitz continuous, which is a common assumption in control theory and allows a rich class of control policies including various types of neural networks~\cite{SzegedyZSBEGF13}. Moreover, assuming Lipschitz continuity is standard in existing works on stability analysis~\cite{RichardsB018,ChangRG19}. Finally, we assume that $d$ is a product of independent univariate distributions, which is needed for efficient sampling and expected value computation.

\paragraph{Probability space of trajectories} A sequence of state-action-disturbance triples $(\mathbf{x}_t,\mathbf{u}_t,\omega_t)_{t\in\mathbb{N}_0}$ is said to be a trajectory of the system, if for each $t\in\mathbb{N}_0$ we have $\mathbf{u}_t=\pi(\mathbf{x}_t)$, $\omega_t\in\support(d)$ and $\mathbf{x}_{t+1}=f(\mathbf{x}_t,\mathbf{u}_t,\omega_t)$. For each initial state $\mathbf{x}_0\in\mathcal{X}$, the system dynamics induces a Markov process which gives rise to the probability space over the set of all trajectories that start in $\mathbf{x}_0$~\cite[Section 2]{Puterman94}. We use $\mathbb{P}_{\mathbf{x}_0}$ and $\mathbb{E}_{\mathbf{x}_0}$ to denote the probability measure and the expectation operator in this probability space.

\paragraph{Almost-sure (a.s.) asymptotic stability} There are several notions of stochastic stability, so we formally define the one that we consider in this work~\cite{kushner1965stability}. Consider a stochastic feedback loop system defined as above, and let $\Stable\subseteq\mathcal{X}$ be Borel-measurable. We say that $\Stable$ is {\em closed under system dynamics} if, for every $\mathbf{x}\in\Stable$ and $\omega\in\support(d)$, we have that $f(\mathbf{x},\pi(\mathbf{x}),\omega)\in\Stable$.

For $\Stable\subseteq\mathcal{X}$ that is closed under system dynamics, we say that it is almost-surely asymptotically stable if from any initial state the system almost-surely converges to $\Stable$ (and therefore stays in $\Stable$ due to closedness under system dynamics).
In order to define this formally, for each $\mathbf{x}\in\mathcal{X}$ let $d(\mathbf{x},\Stable)=\inf_{\mathbf{x_s}\in\Stable}||\mathbf{x}-\mathbf{x}_s||_1$, where $||\cdot||_1$ is the $l_1$-norm on the Euclidean space $\mathbb{R}^m$.

\begin{definition}\label{def:stability}
A non-empty set of states $\Stable\subseteq\mathcal{X}$ that is closed under system dynamics is said to be {\em almost-surely (a.s.) asymptotically stable} if, for each $\mathbf{x}_0\in\mathcal{X}$, we have
\[ \mathbb{P}_{\mathbf{x}_0}\Big[ \lim_{t\rightarrow\infty}d(\mathbf{x}_t,\Stable) = 0 \Big] = 1. \]
\end{definition}

Our definition slightly differs from that of~\cite{kushner1965stability} which considers the special case of the stabilization set being the singleton equilibrium point at the origin, i.e.~$\Stable=\{\mathbf{0}\}$. The reason for this discrepancy is that many practical approaches to the stability analysis of nonlinear systems need to make additional assumptions on the system's behavior around the origin, as otherwise they would suffer from numerical error issues. For instance, \cite{BerkenkampTS017,RichardsB018} study stability of deterministic dynamical systems and both assume that some open neighbourhood of the origin is {\em a priori} known to be stable, whereas~\cite{ChangRG19} only check stability conditions away from some neighbourhood around the origin. In order to avoid making such assumptions and to ensure that our method truly certifies stability, we assume that the region $\mathcal{X}_s$ has {\em non-empty interior} (i.e.~it contains an open ball around a point in $\mathcal{X}$). By making either of the assumptions from the aforementioned works, our method naturally extends to the case where $\mathcal{X}_s=\{\mathbf{0}\}$.

\paragraph{Relation to a.s.~reachability verification} We remark that our method can also formally verify a.s.~reachability of a specified target set, i.e.~that for any initial state the system reaches a state in the target set with probability $1$. In fact, due to the assumption that the stabilization set $\Stable$ is closed under system dynamics, the problem of verifying a.s. asymptotic stability reduces to the a.s. reachability verification problem for the stabilization set.

Assuming the closedness under system dynamics of the stabilization set is a reasonable and a realistic choice, due to dynamical systems typically expressing weak dynamics around the systems' stable points.
As discussed above, many works on stability of deterministic dynamical systems also make a similar assumption, i.e. that an open
neighbourhood of the origin $\mathbf{0}$ is closed under system dynamics~\cite{BerkenkampTS017,RichardsB018,ChangRG19}.


\section{Theoretical Results}

We now present a theoretical framework for formally certifying stability of a region in a discrete-time stochastic dynamical system. Our framework is based on {\em ranking supermartingales} which we introduce below.


\paragraph{Ranking supermartingales} Consider a discrete-time stochastic dynamical system defined by a dynamics function $f$, a policy $\pi$ and a probability distribution $d$ with model assumptions as in the previous section, and let $\Stable\subseteq\mathcal{X}$ be closed under system dynamics and have non-empty interior.

Intuitively, a ranking supermartingale (RSM) is a nonnegative continuous function whose value at each state in $\mathcal{X}\backslash \Stable$ decreases in expectation by at least $\eps>0$ (is {\em $\eps$-ranked}) after a one-step evolution of the system under the policy $\pi$, where the expected value is taken with respect to the probability distribution $d$ over disturbance vectors. The name comes from the connection to supermartingales, a class of discrete-time stochastic processes in probability theory whose value decreases in expectation after each time step~\cite{Williams91}. RSMs were first introduced in~\cite{ChakarovS13} for the termination analysis of probabilistic programs, and we adapt them to the setting of stochastic dynamical systems.

\begin{definition}\label{def:stochasticlyapunov}
A continuous function $V: \mathcal{X} \rightarrow \mathbb{R}$ is said to be a {\em ranking supermartingale (RSM) for $\Stable$}, if $V(\mathbf{x})\geq 0$ holds for any $\mathbf{x}\in\mathcal{X}$ and if there exists $\eps > 0$ such that
\begin{equation}\label{eq:expdec}
    \mathbb{E}_{\omega\sim d}\Big[ V \Big( f(\mathbf{x}, \pi(\mathbf{x}), \omega) \Big) \Big] \leq V(\mathbf{x}) - \eps
\end{equation}
holds for every $\mathbf{x}\in \mathcal{X}\backslash \Stable$.
\end{definition}

We note that RSMs differ from the commonly considered stochastic Lyapunov functions for discrete-time stochastic systems~\cite{Kushner14}, which require $V$ to be continuous and to satisfy the following conditions:
\begin{itemize}
    \item $\mathbb{E}_{\omega\sim d}[ V( f(\mathbf{x}, \pi(\mathbf{x}), \omega) )] < V(\mathbf{x})$ for $\mathbf{x}\in \mathcal{X}\backslash \Stable$,
    \item $V(\mathbf{x}) > 0$ for $\mathbf{x}\in \mathcal{X}\backslash \Stable$, and
    \item $V(\mathbf{x}) = 0$ for $\mathbf{x}\in \Stable$.
\end{itemize}
The third condition would be quite restrictive if we tried to learn $V$ in the form of a neural network (which will be the goal of our novel approach to stability verification in the next section). Thus, one of the key benefits of considering RSMs instead of stochastic Lyapunov functions is that we may replace the $V(\mathbf{x}) = 0$ for $\mathbf{x}\in \Stable$ condition by a slightly stricter expected decrease condition that requires the decrease by at least some $\eps>0$. Theorem~\ref{thm:soundness} establishes that RSMs are indeed sufficient to prove a.s.~asymptotic stability.

\begin{theorem}\label{thm:soundness}
Let $f:\mathcal{X}\times\mathcal{U} \times \mathcal{N}\rightarrow \mathcal{X}$ be a Lipschitz continuous dynamics function, $\pi:\mathcal{X}\rightarrow\mathcal{U}$ a Lipschitz continuous policy and $d$ a distribution over $\mathcal{N}$. Suppose that $\mathcal{X}$ is compact and let $\Stable\subseteq\mathcal{X}$ be closed under system dynamics and have a non-empty interior. Suppose that there exists an RSM $V:\mathcal{X}\rightarrow\mathbb{R}$ for $\Stable$. Then $\Stable$ is a.s.~asymptotically stable.
\end{theorem}

The main idea behind the proof of Theorem~\ref{thm:soundness} is as follows. For each state $\mathbf{x}_0\in\mathcal{X}$, we consider the probability space of all trajectories that start in $\mathbf{x}_0$. We then show that the RSM $V$ for $\Stable$ gives rise to an instance of the mathematical notion of RSMs in this probability space, and use results from probability theory on the convergence of RSMs to conclude that $\Stable$ is a.s.~asymptotically stable. The overview of the results from probability and martingale theory that we use in our proof as well as the formal proof of the theorem can be found in the Supplementary Material.

\paragraph{Bounds on the convergence time} While formally verifying that a control policy stabilizes the system with probability $1$ is very important for safety critical applications, another practical concern is to ensure that stabilization happens within some tolerable time limit. 

Another important caveat of using RSMs for stability analysis of stochastic systems is that they provide formal guarantees on the stabilization time. For a system trajectory $(\mathbf{x}_t,\mathbf{u}_t,\omega_t)_{t\in\mathbb{N}_0}$, we define its stabilization time $T_{\Stable} = \inf\{t\in\mathbb{N}_0 \mid \mathbf{x}_t\in\Stable \}$ to be the first hitting time of the region $\Stable$ (with $T_{\Stable}=\infty$ if trajectory never reaches $\Stable$). Given $c>0$, the system has {\em $c$-bounded differences} if the distance between any two consecutive system states with respect to the $l_1$-norm does not exceed $c$, i.e.~for any $\mathbf{x}\in\mathcal{X}$ and $\omega\in\support(d)$ we have $||f(\mathbf{x},\pi(\mathbf{x}),\omega) - \mathbf{x}||_1 \leq c$.

\begin{theorem}\label{thm:bounds}
Let $f:\mathcal{X}\times\mathcal{U} \times \mathcal{N}\rightarrow \mathcal{X}$ be a Lipschitz continuous dynamics function, $\pi:\mathcal{X}\rightarrow\mathcal{U}$ a Lipschitz continuous policy and $d$ a distribution over $\mathcal{N}$. Suppose that $\mathcal{X}$ is compact and let $\Stable\subseteq\mathcal{X}$ be closed under system dynamics and have a non-empty interior. Suppose that there exists an $\eps$-RSM $V:\mathcal{X}\rightarrow\mathbb{R}$ for $\Stable$. Then, for any initial state $\mathbf{x}_0\in\mathcal{X}$,
\begin{enumerate}
    \item $\mathbb{E}_{\mathbf{x}_0}[ T_{\Stable} ] \leq \frac{V(\mathbf{x}_0)}{\eps}$.
    \item $\mathbb{P}_{\mathbf{x}_0}[ T_{\Stable} \geq t ] \leq \frac{V(\mathbf{x}_0)}{\eps \cdot t}$, for any time $t\in\mathbb{N}$.
    \item If the system has {\em $c$-bounded differences} for $c>0$, then $\mathbb{P}_{\mathbf{x}_0}[ T_{\Stable} \geq t ] \leq A\cdot e^{-t\cdot \eps^2/(2\cdot (c+\eps)^2)}$ for any time $t\in\mathbb{N}$ and $A=e^{\eps\cdot V(\mathbf{x}_0)/(c+\eps)^2}$.
\end{enumerate}
\end{theorem}

The proof of Theorem~\ref{thm:bounds} can be found in the Supplementary Material and here we present the key ideas. The first part shows that the expected stabilization time is bounded from above by the initial value of $V$ divided by $\eps$. To prove it, we show that the stabilization time gives rise to a stopping time in the probability space of all trajectories that start in $\mathbf{x}_0$. We then observe that the RSM $V$ satisfies the expected decrease condition until $T_{\Stable}$ is exceeded and use the results from probability theory on the convergence of RSMs to conclude the bound on the expected value of this stopping time. 

The second part shows a bound on the probability that the stabilization time exceeds a threshold $t$, and it follows immediately from the first part by an application of Markov's inequality. Note that this bound decays linearly in $t$, as $t\rightarrow\infty$.

Finally, the third part shows an asymptotically tighter bound with the decay in $t$ being exponential, for systems that have $c$-bounded differences. The proof follows by an application of Azuma's inequality~\cite{azuma1967weighted} which is a classical result from martingale theory and which we also include in the Supplementary Material.

\section{Method for Stability Verification}

In this section, we present our method for verifying a.s.~asymptotic stability of a given region via RSM computation. Our method consists of two modules which alternate within a loop: the {\em learner} and the {\em verifier}. In each loop iteration, the learner first learns an RSM candidate in the form of a neural network. The candidate is then passed to the verifier, which checks whether the learned candidate is indeed an RSM. If the answer is positive, the verifier terminates the loop and concludes the system's a.s.~asymptotic stability. Otherwise, the verifier computes a set of counterexamples which show that the candidate is not an RSM and passes it to the learner, which then proceeds with the next learning iteration. This process is repeated until either a learned candidate is verified or a given timeout is reached.

We consider a discrete-time stochastic dynamical system defined by a dynamics function $f$, a policy $\pi$ and a probability distribution $d$ with model assumptions as in the previous sections, and $\Stable\subseteq\mathcal{X}$ which is closed under system dynamics and has non-empty interior. The rest of this section describes the details behind our method for stability verification. The algorithm is presented in Algorithm~\ref{alg:algorithm}.

\subsection{Discretization and initial sampling} Recall, an RSM $V$ needs to satisfy the expected decrease condition in eq.~(\ref{eq:expdec}) at each point in $\mathcal{X}\backslash\Stable$. However, one of the main difficulties in verifying this condition when $V$ has a neural network form is that it is not clear how to compute a closed form for the expected value of $V$ at a successor system state. In order to be able to verify neural network RSM candidates, our method discretizes the state space and then verifies the expected decrease condition only at the states in the discretization (which we will show to be possible due to $f$, $\pi$ and $V$ all being Lipschitz continuous and $\mathcal{X}$ being compact). The {\em discretization} $\tilde{\mathcal{X}}$ of $\mathcal{X}\backslash \Stable$ satisfies the property that, for each $\mathbf{x}\in\mathcal{X}$, there is $\tilde{\mathbf{x}}\in\tilde{\mathcal{X}}$ with $||\mathbf{x}-\tilde{\mathbf{x}}||_1<\tau$, with $\tau$ an algorithm parameter that we call the {\em mesh} of $\tilde{X}$. Since $\mathcal{X}$ is compact and so $\mathcal{X}\backslash \Stable$ is bounded, the discretization consists of finitely many states.

The method also initializes the collection of pairs $\mathcal{D} = \{ (\mathbf{x},\mathcal{D}_{\mathbf{x}}) \mid \mathbf{x}\in\tilde{\mathcal{X}}\}$, where each $\mathcal{D}_{\mathbf{x}}$ consists of $N$ successor states of $\mathbf{x}$ obtained by independent sampling. Here, $N\in\mathbb{N}$ is an algorithm parameter. The collection $\mathcal{D}$ will be used to approximate expected values at successor states for each $\mathbf{x}$ in $\tilde{\mathcal{X}}$ in the loss function used by the learner.

\begin{algorithm}[t]
\caption{Verification of a.s.~asymptotic stability}
\label{alg:algorithm}
\begin{algorithmic}
\STATE \textbf{Input} Dynamics function $f$, policy $\pi$, disturbance distribution $d$, region $\Stable\subseteq\mathcal{X}$, Lipschitz constants $L_f$, $L_{\pi}$ \\
parameters $\tau>0$, $N\in\mathbb{N}$, $\lambda > 0$

\STATE $\tilde{\mathcal{X}} \leftarrow $ discretization of $\mathcal{X}\backslash \Stable$ with mesh $\tau$
\FOR{$\mathbf{x}$ in $\tilde{\mathcal{X}}$}
\STATE $\mathcal{D}_{\mathbf{x}}\leftarrow$ $N$ sampled successor states of $\mathbf{x}$
\ENDFOR
\WHILE{timeout not reached}
\STATE $V \leftarrow$ trained candidate by minimizing the loss in eq.~(\ref{eq:loss})
\STATE $L_V \leftarrow$ Lipschitz constant of $V$
\STATE $K\leftarrow L_V \cdot (L_f \cdot (L_\pi + 1) + 1)$
\IF{$\exists \mathbf{x}\in\tilde{\mathcal{X}}$ s.t. $\mathbb{E}_{\omega\sim d}[ V ( f(\mathbf{x}, \pi(\mathbf{x}), \omega) ) ] \geq V(\mathbf{x}) - \tau \cdot K$}
\STATE $\mathcal{D}_{\mathbf{x}}\leftarrow$ add $N$ sampled successor states of $\mathbf{x}$
\ELSE
\STATE \textbf{Return} A.s.~asymptotically stable
\ENDIF
\ENDWHILE
\STATE \textbf{Return} Unknown
\end{algorithmic}
\end{algorithm}

\subsection{Verifier}

In order to motivate the form of the loss function used by the learner, we first describe the verifier module of our algorithm. For a neural network $V$ to be an RSM as in Definition~\ref{def:stochasticlyapunov}, it needs to be (1)~continuous, (2)~nonnegative at each state, and (3)~to satisfy the expected decrease condition in eq.~(\ref{eq:expdec}) for each state in $\mathcal{X}\backslash\Stable$. Since $V$ is a neural network we already know that it is a continuous function. Moreover, since $\mathcal{X}$ is compact and $V$ is continuous, the function $V$ admits a finite global lower bound $-m\in\mathbb{R}$. Hence, if we verify that $V$ satisfies the expected decrease condition, we may consider the function $V'(\mathbf{x})=V(\mathbf{x})+m$ which is in addition nonnegative and thus an RSM to conclude a.s.~asymptotic stability of $\Stable$. Therefore, the verifier only needs to check that $V$ satisfies the expected decrease condition in eq.~(\ref{eq:expdec}) for each state in $\mathcal{X}\backslash\Stable$, from which it immediately follows that $V'$ is an RSM.

As explained above, checking this for each state in $\mathcal{X}\backslash\Stable$ is not feasible since we cannot compute a closed form for the expected value of $V$ at a successor system state. Instead, we show that it is sufficient to check a slightly stricter condition on states in the discretization $\tilde{\mathcal{X}}$. Let $L_f$, $L_\pi$ and $L_V$ be the Lipschitz constants of $f$, $\pi$ and the candidate function $V$, respectively. We assume that the Lipschitz constant for the dynamics function $f$ and the policy $\pi$ are provided, and use the method of~\cite{SzegedyZSBEGF13} to compute the Lipschitz constant of the neural network candidate $V$ (and also of $\pi$, in cases when $\pi$ is a neural network policy). Then define
\begin{equation}\label{eq:k}
    K = L_V \cdot (L_f \cdot (L_\pi + 1) + 1).
\end{equation}
In order to verify that $V$ satisfies the expected decrease condition in eq.~(\ref{eq:expdec}) for each state in $\mathcal{X}\backslash\Stable$, the verifier checks for each $\mathbf{x}$ in the discretization $\tilde{\mathcal{X}}$ that
\begin{equation}\label{eq:lipschitz}
    \mathbb{E}_{\omega\sim d}\Big[ V \Big( f(\mathbf{x}, \pi(\mathbf{x}), \omega) \Big) \Big] < V(\mathbf{x}) - \tau \cdot K
\end{equation}
If eq.~(\ref{eq:lipschitz}) holds for each $\mathbf{x}\in \tilde{\mathcal{X}}$, the verifier concludes a.s.~asymptotic stability of $\Stable$. Otherwise, if $\mathbf{x}\in\tilde{\mathcal{X}}$ for which eq.~(\ref{eq:lipschitz}) does not hold is found, it is passed to the learner by independently sampling $N$ successor states of $\mathbf{x}$ which are then added to the set $\mathcal{D}_{\mathbf{x}}$.

Theorem~\ref{theorem:lipschitz} establishes the correctness of Algorithm~\ref{alg:algorithm} by showing that it indeed suffices to check eq.~(\ref{eq:lipschitz}) for states in the discretization. The proof uses the fact that $f$ and $\pi$ are Lipschitz continuous and that $\mathcal{X}$ is compact, and is provided in the Supplementary Material.

\begin{theorem}\label{theorem:lipschitz}
Suppose that the verifier in Algorithm~\ref{alg:algorithm} verifies that $V$ satisfies eq.~(\ref{eq:lipschitz}) for each $\mathbf{x}\in\tilde{\mathcal{X}}$. Let $-m\in\mathbb{R}$ be such that $V(\mathbf{x})\geq -m$ for each $\mathbf{x}\in\mathcal{X}$. Then, the function $V'(\mathbf{x})=V(\mathbf{x})+m$ is an RSM for $\Stable$. Hence, $\Stable$ is a.s.~asymptotically stable.
\end{theorem}

We remark that the cardinality of the discretization $\tilde{\mathcal{X}}$ grows exponentially in the dimension of the state space, which in turn implies an exponential complexity for each verification step in our algorithm. This limitation is also present in related works on stability analysis in deterministic dynamical systems~\cite{BerkenkampTS017}. A potential approach to overcome the complexity bottleneck would be to discretize different dimensions and regions of the state space with a heterogeneous instead of a uniform granularity.

\paragraph{Expected value computation} What is left to be described is how our algorithm computes the expected value in eq.~(\ref{eq:lipschitz}) for a given state $\mathbf{x}\in\tilde{\mathcal{X}}$. This is {\em not} trivial, since $V$ is a neural network and so we do not have a closed form for the expected value. 
However, we can bound the expected value via interval arithmetic. In particular, let $\mathbf{x}\in\tilde{\mathcal{X}}$ be a throughout fixed state for which we want to bound the expected value $\mathbb{E}_{\omega\sim d}[ V ( f(\mathbf{x}, \pi(\mathbf{x}), \omega) )]$. Our algorithm partitions the disturbance space $\mathcal{N}\subseteq\mathbb{R}^p$ into finitely many cells $\text{cell}(\mathcal{N}) = \{\mathcal{N}_1,\dots,\mathcal{N}_{k}\}$, with $k$ being the number of cells.
We use $\mathrm{maxvol}=\max_{\mathcal{N}_i\in \text{cell}(\mathcal{N})}\mathsf{vol}(\mathcal{N}_i)$ to denote the maximal volume with respect to the Lebesgue measure over $\mathbb{R}^p$ of any cell in the partition. The algorithm then bounds the expected value via
\begin{equation*}
    \mathbb{E}_{\omega\sim d}\Big[ V \Big( f(\mathbf{x}, \pi(\mathbf{x}), \omega) \Big) \Big] \leq \sum_{\mathcal{N}_i\in \text{cell}(\mathcal{N})} \mathrm{maxvol} \cdot \sup_{\omega\in \mathcal{N}_i} F(\omega)
\end{equation*}
where $F(\omega) = V( f(\mathbf{x}, \pi(\mathbf{x}), \omega))$. Each supremum is then bounded from above via interval arithmetic by using the method of~\cite{Gowal18}. In our experimental evaluation, we observed that this method computes very tight bounds when the number of cells is sufficiently large.

Note that $\mathrm{maxvol}$ is not finite in cases when $\mathcal{N}$ is unbounded. In order to allow expected value computation for an unbounded $\mathcal{N}$, we first apply the probability integral transform~\cite{Murphy12} to each univariate probability distribution in $d$. Recall, in our model assumptions we assumed that $d$ is a product of univariate distributions and our dynamics function $f$ takes the most general form.

\subsection{Learner}

We now describe the learner module of our algorithm. The learner constructs an RSM candidate function as a multi-layer neural network $V_{\theta}$, where $\theta$ is the vector of neural network parameters. A candidate neural network is learned by minimizing the following loss function
\begin{equation}\label{eq:loss}
    \mathcal{L}(\theta) = \mathcal{L}_{\text{RSM}}(\theta) + \lambda \cdot  \mathcal{L}_{\text{Lipschitz}}(\theta).
\end{equation}
The first loss term $\mathcal{L}_{\text{RSM}}(\theta)$ is defined via
\begin{equation*}
    \mathcal{L}_{\text{RSM}}(\theta) = \frac{1}{|\tilde{\mathcal{X}}|}\sum_{\mathbf{x}\in\tilde{\mathcal{X}}}\Big( \max\Big\{\sum_{\mathbf{x}'\in\mathcal{D}_{\mathbf{x}}}\frac{V_{\theta}(\mathbf{x}')}{|\mathcal{D}_{\mathbf{x}}|}-V_{\theta}(\mathbf{x}) + \tau \cdot K, 0\Big\} \Big).
\end{equation*}
Intuitively, for $\mathbf{x}\in\tilde{\mathcal{X}}$, the corresponding term in the sum incurs a loss whenever the condition in eq.~(\ref{eq:lipschitz}) is violated. Since the closed form for the expected value in eq.~(\ref{eq:lipschitz}) in terms of parameters $\theta$ cannot be computed, for each $\mathbf{x}\in\tilde{\mathcal{X}}$ we approximate it as the mean of the values of $V$ at sampled successor states of $\mathbf{x}$ that the algorithm stores in the set $\mathcal{D}_{\mathbf{x}}$.

The second loss term $\lambda \cdot  \mathcal{L}_{\text{Lipschitz}}(\theta)$ is the regularization term used to incentivize that the Lipschitz constant $L_{V_{\theta}}$ of $V_{\theta}$ does not exceed some tolerable threshold, and hence to enforce that $\tau\cdot K$ in eq.~(\ref{eq:lipschitz}) is sufficiently small. The constant $\lambda$ is an algorithm parameter balancing the two loss terms, and we define
\[ \mathcal{L}_{\text{Lipschitz}}(\theta) = \max\Big\{ \frac{\delta}{\tau \cdot (L_f \cdot (L_\pi + 1) + 1)} - L_{V_{\theta}}, 0 \Big\}. \]
Here, $\delta$ is a parameter that specifies the threshold, and $L_{V_{\theta}}$ in terms of $\theta$ is computed as in~\cite{SzegedyZSBEGF13}.

To conclude this section, we note that the loss function $\mathcal{L}(\theta)$ is nonnegative but is not necessarily equal to $0$ even if $V_{\theta}$ satisfies eq.~(\ref{eq:lipschitz}) for each $\mathbf{x}\in\tilde{\mathcal{X}}$ and its Lipschitz constant is below the allowed threshold. This is because $\mathcal{L}(\theta)$ depends on samples in $\mathcal{D}$ which are used to {\em approximate} the expected values in eq.~(\ref{eq:lipschitz}). However, in Theorem~\ref{thm:loss} we show that the loss $\mathcal{L}(\theta) \rightarrow 0$ almost-surely as we add samples to the set $\mathcal{D}_{\mathbf{x}}$ for each  $\mathbf{x}\in\tilde{\mathcal{X}}$, whenever $V_{\theta}$ satisfies eq.~(\ref{eq:lipschitz}) for each $\mathbf{x}\in\tilde{\mathcal{X}}$ and its Lipschitz constant is below the allowed threshold. The claim follows from the Strong Law of Large Numbers~\cite[Section 12.10]{Williams91} and the proof can be found in the Supplementary Material.

\begin{theorem}\label{thm:loss}
Let $M = \min_{\mathbf{x}\in\tilde{\mathcal{X}}}|\mathcal{D}_{\mathbf{x}}|$. If $V_{\theta}$ satisfies eq.~(\ref{eq:lipschitz}) for each $\mathbf{x}\in\tilde{\mathcal{X}}$ and if $L_{V_{\theta}} \leq \delta / (\tau\cdot (L_f \cdot (L_\pi + 1) + 1))$, then $\lim_{M\rightarrow \infty} \mathcal{L}(\theta) = 0$ holds almost-surely.
\end{theorem}

\begin{figure}[t]
    \centering
    \includegraphics[width=0.5\textwidth]{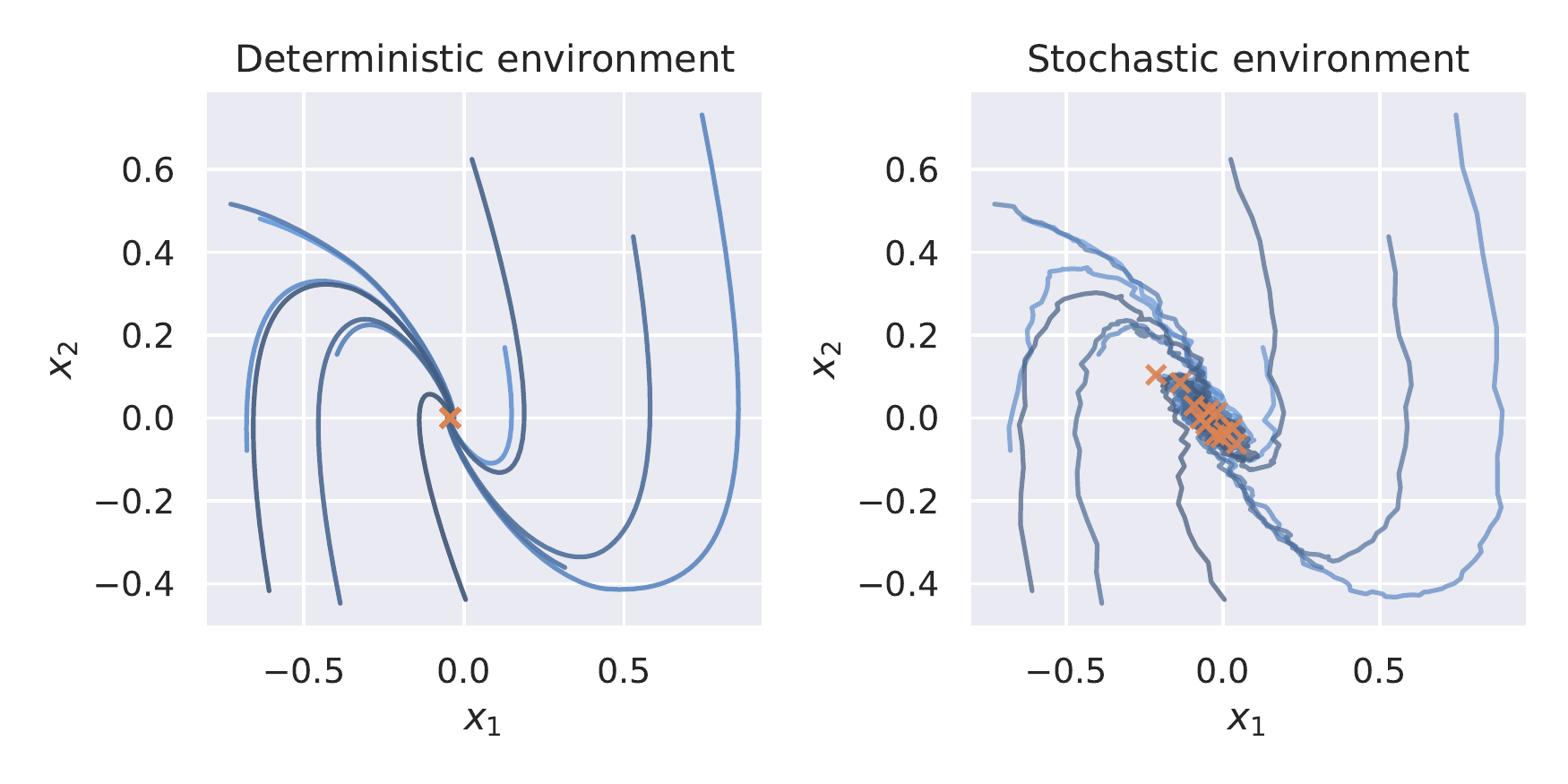}
    \caption{Example of a deterministic and a stochastic system with the same dynamics function, illustrating the difficulties of proving stability in stochastic systems. The orange markers indicate the system state after 200 time steps.}
    \label{fig:environment}
\end{figure}

\subsection{Learning stable policies}

While in this work we focus on the stability verification problem for a given control policy, our approach can also be adapted to the setting in which we want to {\em learn} a stable neural network policy for the region $\Stable$ together with a formal certificate for the a.s.~asymptotic stability of $\Stable$. This can be done by replacing the loss function in eq.~(\ref{eq:loss}) with
\[ \mathcal{L}(\theta, \mathbf{u}) = \mathcal{L}_{\text{RSM}}(\theta, \mathbf{u}) + \lambda \cdot  \mathcal{L}_{\text{Lipschitz}}(\theta, \mathbf{u}) \]
where $\mathbf{u}$ is now a vector of policy parameters while $\theta$ is again a vector of neural network parameters for the RSM candidate. The correctness of our algorithm proved in Theorem~\ref{theorem:lipschitz} then ensures that any learned and verified control policy is indeed stable. Note that this modified algorithm does not try to optimize the expected reward obtained by the learned policy, but only ensures stability. Exploring ways to learn a stable policy while simultaneously maximizing the expected reward is an interesting direction of future work.

\section{Experiments}

\begin{figure}[t]
    \centering
    \includegraphics[width=0.5\textwidth]{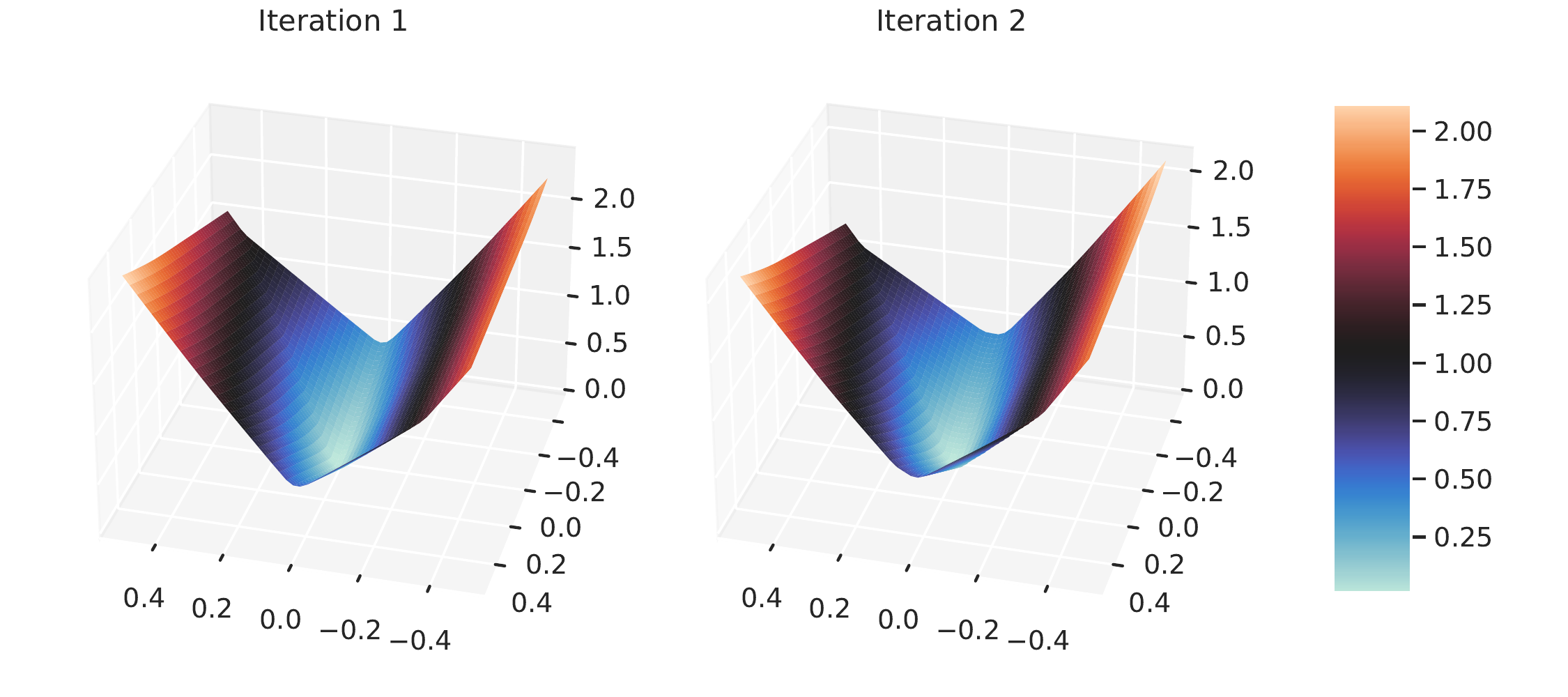}
    \caption{Learned RSM candidates after 1 and 2 iterations of our algorithm for the stochastic inverted pendulum task. The candidate on the left violates the expected decrease condition while the function of the right is a verified RSM.}
    \label{fig:rsm}
\end{figure}

We validate our algorithm empirically on two RL benchmark environments. 
Our first benchmark is a two-dimensional dynamical system of the form $\mathbf{x}_{t+1} = A\mathbf{x}_t + B g(\mathbf{u}_t) + \omega$, where $\omega$ is a disturbance vector sampled from a zero-mean triangular distribution. The function $g$ clips the control action to stay within the interval $[-1,1]$. The matrices $A$ and $B$ are provided in the Supplementary Material.

Our second benchmark is the inverted pendulum problem \cite{gym}.
Contrarily to the standard inverted pendulum task, which has deterministic dynamics, we consider a more difficult stochastic variant. The system has two state variables $x_1$ and $x_2$ which represent the angle and the angular velocity of the pendulum. The objective of this task is to balance the pendulum in an upright position through control actions in the form of a torque that is applied to the pendulum. Our stochastic variant of the task applies a zero-mean triangular noise to both state variables.

For each RL task, we consider the state space $\mathcal{X}=\{\mathbf{x} \mid ||\mathbf{x}||_1 \leq 0.5\}$ and train a control policy comprised of two hidden layers with 128 ReLU units each by using proximal policy optimization \cite{schulman2017proximal}, while applying our Lipschitz regularization to keep the Lipschitz constant of the policy within a reasonable bound. We then run our algorithm to verify that the region $\Stable = \{\mathbf{x} \mid ||\mathbf{x}||_1 \leq 0.2\}$ is a.s.~asymptotically stable. Our RSM neural networks consist of one hidden layer with 128 ReLU units.

Example trajectories of a policy trained for the first benchmark with the deterministic ($\omega =0$) and stochastic dynamics are shown in Figure \ref{fig:environment}. The policy stabilizes the deterministic system in a single point, however this is not the case for the stochastic system. This illustrates the intricacies of verifying stability in stochastic systems, and justifies our choice to consider stabilizing regions with non-empty interiors.


Our method could successfully learn and verify RSMs for both systems within a reasonable time frame. The runtime statistics are shown in Table \ref{tab:stats}. The final RSM neural network for the inverted pendulum task is shown in Figure~\ref{fig:rsm}.
We further computed the $\eps$ of the RSM network according to Definition \ref{def:stochasticlyapunov} for the inverted pendulum task to obtain the convergence time bounds as outlined in Theorem \ref{thm:bounds}. The resulting convergence time bounds are shown in Figure \ref{fig:convbounds}.

\begin{figure}
    \centering
    \includegraphics[width=0.30\textwidth]{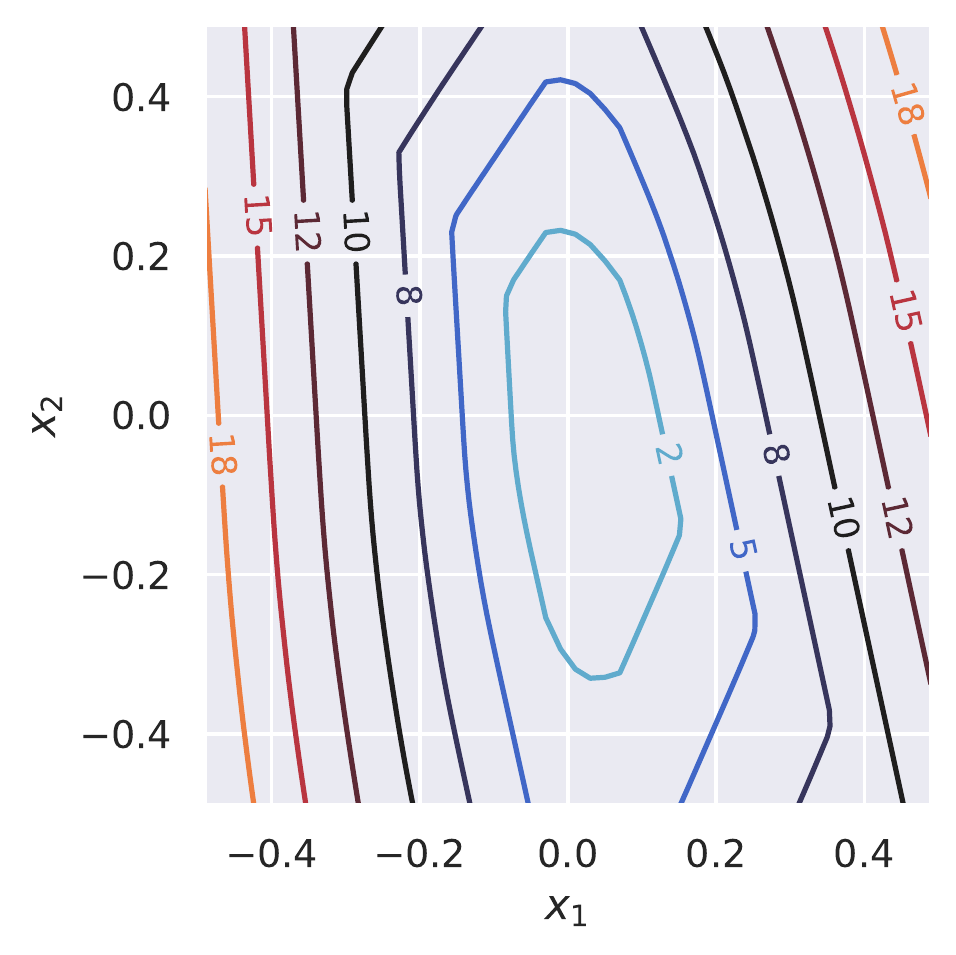}
    \caption{Contour lines of the convergence time bounds obtained from the RSM on the inverted pendulum task.}
    \label{fig:convbounds}
\end{figure}

We perform an additional experiment to study the effectiveness of our method for computing bounds on the expected value a neural network. In particular, we sample $100$ random states of the inverted pendulum environment. For each sampled state, we use our method to compute the bound on the expected value of the final RSM neural network (shown in Figure \ref{fig:rsm}) in a successor system state, with different sizes of the cell partition. We then compute the ground-truth of the expected value by averaging the RSM value at 1000 independently sampled successor states (Strong Law of Large Numbers). The results shown in Figure \ref{fig:integral} indicate that, even with a modest size of the cell partition, a tight bound can be obtained. As the partition is further refined, the expected value bound converges to the ground-truth.

\begin{table}[]
    \centering
    \begin{tabular}{c|ccc}\toprule
    Environment & Iters. & Mesh ($\tau$) & Runtime \\\midrule
        2D system & 4 & 0.002 & 559 \\
        Inverted pendulum & 2 & 0.01 & 176 \\\bottomrule
    \end{tabular}
    \caption{Number of learner-verifier loop iterations, mesh of the discretization used by the verifier, and the total algorithm runtime (in seconds).}
    \label{tab:stats}
\end{table}

\begin{figure}[t]
    \centering
    \includegraphics[width=0.5\textwidth]{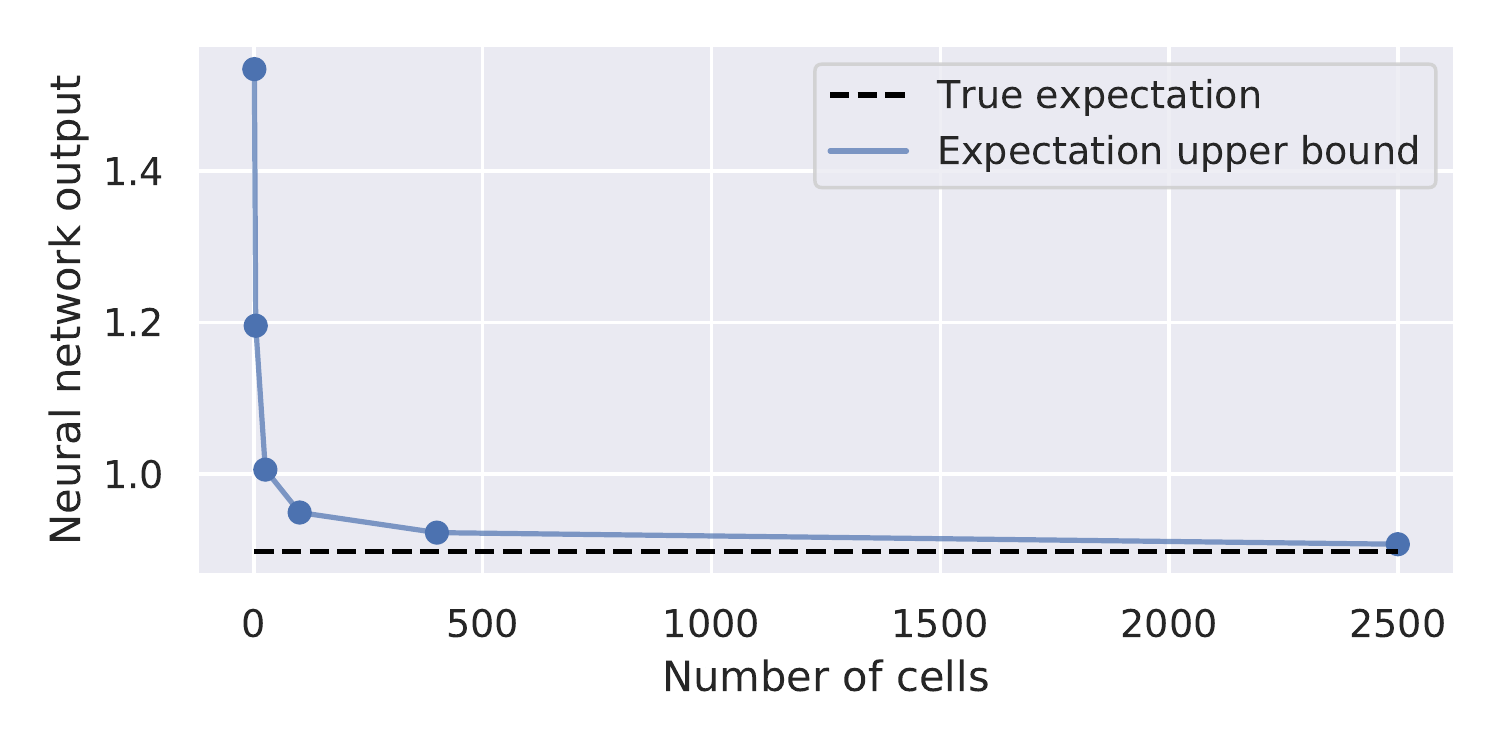}
    \caption{Comparison of our method for bounding the expected value of an RSM neural network with the ground-truth expected value on 100 randomly sampled states of the inverted pendulum environment. }
    \label{fig:integral}
\end{figure}

\section{Conclusion}

In this work, we study the stability verification problem for nonlinear stochastic control systems. We show, for the first time, that ranking supermartingales (RSMs) provide a formal certificate for a.s.~asymptotic stability as well as guarantees on the stabilization time. We then present a method for a.s.~asymptotic stability verification which learns and verifies an RSM in the form of a neural network. In order to design the verifier module of our algorithm, we propose a method for efficiently computing tight bounds on the expected value of a neural network function over a probability distribution. Finally, we validate our approach experimentally on a set of nonlinear stochastic RL environments with neural network policies. There are several interesting venues for future work. While we showed how our verification algorithm can be adapted to also learn a stabilizing policy, this adaptation does not try to optimize the learned policy. Exploring ways to learn high performing stabilizing policies is an interesting direction.
A limiting factor of our algorithm for computing RSMs is that the complexity of the verification step grows exponentially with the dimension of the state space. In order to overcome this limitation and improve scalability, future work may consider different ways to discretize the state space, such as an on-demand discretization that does not use the same granularity at all parts of the state space. 
Another future research direction is to integrate our approach to safe exploration RL in systems with stochastic environments.

\section{Acknowledgement}
This work was supported in part by the ERC-2020-AdG 101020093, ERC CoG 863818 (FoRM-SMArt) and the European Union’s Horizon 2020 research and innovation programme under the Marie Skłodowska-Curie Grant Agreement No.~665385.

\bibliography{aaai22stability.bib}

\clearpage
\begin{center}
\textbf{\Large Supplementary Material}
\end{center}

\section{Proofs of Theorem~1 and Theorem~2}

In this section, we provide the proofs of Theorem~1 and Theorem~2 from the main text of the paper. Since these proofs heavily rely on results from martingale theory, we first provide an overview of relevant definitions and results. We then proceed to proving the theorems.

\subsection{An Aside on Martingale Theory}\label{sec:martingales}

\paragraph{Probability space} A {\em probability space} is a triple $(\Omega,\mathcal{F},\mathbb{P})$, where $\Omega$ is a non-empty {\em sample space}, $\mathcal{F}$ is a $\sigma$-algebra over $\Omega$ (i.e.~a collection of subsets of $\Omega$ that contains the empty set $\emptyset$ and is closed under complementation and countable union operations), and $\mathbb{P}$ is a {\em probability measure} over $\mathcal{F}$, i.e.~a function $\mathbb{P}:\mathcal{F}\rightarrow[0,1]$ that satisfies the following properties: (1)~$\mathbb{P}[\emptyset]=0$, (2)~$\mathbb{P}[\Omega\backslash A]=1-\mathbb{P}[A]$ for each $A\in\mathcal{F}$, and (3)~$\mathbb{P}[\cup_{i=0}^\infty A_i]=\sum_{i=0}^\infty\mathbb{P}[A_i]$ for any sequence $(A_i)_{i=0}^\infty$ of pairwise disjoint sets in $\mathcal{F}$.

\paragraph{Random variable} Given a probability space $(\Omega,\mathcal{F},\mathbb{P})$, a {\em random variable} is a function $X:\Omega\rightarrow\mathbb{R}\cup\{\pm\infty\}$ that is $\mathcal{F}$-measurable, i.e.~for each $a\in\mathbb{R}$ we have that $\{\omega\in\Omega\mid X(\omega)\leq a\}\in\mathcal{F}$. We use $\mathbb{E}[X]$ to denote the {\em expected value} of $X$ (for the formal definition of expected value, see~\cite{Williams91}). A {\em (discrete-time) stochastic process} is a sequence $(X_i)_{i=0}^{\infty}$ of random variables in $(\Omega,\mathcal{F},\mathbb{P})$.
	
\paragraph{Conditional expectation} Let $(\Omega,\mathcal{F},\mathbb{P})$ be a probability space and let $X$ be a random variable in $(\Omega,\mathcal{F},\mathbb{P})$. Given a sub-sigma-algebra $\mathcal{F}'\subseteq\mathcal{F}$, a {\em conditional expectation} of $X$ given $\mathcal{F}'$ is an $\mathcal{F}'$-measurable random variable $Y$ such that, for each $A\in\mathcal{F}'$, we have 
\[ \mathbb{E}[X\cdot\mathbb{I}_A]=\mathbb{E}[Y\cdot\mathbb{I}_A].\]
Here $\mathbb{I}_A:\Omega\rightarrow \{0,1\}$ is an {\em indicator function} of $A$, defined via $\mathbb{I}_A(\omega)=1$ if $\omega\in A$, and $\mathbb{I}_A(\omega)=0$ if $\omega\not\in A$. It is known that a conditional expectation of a random variable $X$ given $\mathcal{F}'$ exists if either (1)~$X$ is {\em integrable}, i.e.~$\mathbb{E}[|X|]<\infty$, or (2)~$X$ is  real-valued and nonnegative~\cite{Williams91} (though these two conditions are not necessary). Moreover, whenever the conditional expectation exists it is also almost-surely unique, i.e.~for any two $\mathcal{F}'$-measurable random variables $Y$ and $Y'$ which are conditional expectations of $X$ given $\mathcal{F}'$, we have that $\mathbb{P}[Y= Y']=1$. Therefore, we may pick any such random variable as a canonical conditional expectation and denote it by $\mathbb{E}[X\mid \mathcal{F}']$.

\paragraph{Stopping time} A sequence of sigma-algebras $\{\mathcal{F}_i\}_{i=0}^{\infty}$ with $\mathcal{F}_0\subseteq \mathcal{F}_1\subseteq\dots\subseteq\mathcal{F}$ is said to be a {\em filtration} in the probability space $(\Omega,\mathcal{F},\mathbb{P})$. A {\em stopping time} with respect to a filtration $\{\mathcal{F}_i\}_{i=0}^{\infty}$ is a random variable $T:\Omega\rightarrow \mathbb{N}_0\cup\{\infty\}$ such that, for every $i\in\mathbb{N}_0$, it holds that $\{\omega\in\Omega\mid T(\omega)\leq i\}\in \mathcal{F}_i$. Intuitively, $T$ returns the time step at which some stochastic process should be ``stopped'', and the fact that $\{\omega\in\Omega\mid T(\omega)\leq i\}\in \mathcal{F}_i$ imposes that the decision on stopping at the time step $i$ is made solely by using the information available in the first $i$ time steps of the process.

\paragraph{Ranking supermartingale} We are finally ready to formally define the mathematical notion of ranking supermartingales. Let $(\Omega,\mathcal{F},\mathbb{P})$ be a probability space and let $\eps\geq 0$. Suppose that $T$ is a stopping time with respect to a filtration $\{\mathcal{F}_i\}_{i=0}^{\infty}$. An {\em $\eps$-ranking supermartingale ($\eps$-RSM) with respect to $T$} is a stochastic process $(X_i)_{i=0}^{\infty}$ such that
\begin{itemize}
    \item $X_i$ is $\mathcal{F}_i$-measurable, for each $i\geq 0$,
    \item $X_i(\omega)\geq 0$, for each $i\geq 0$ and $\omega\in\Omega$, and
    \item $\mathbb{E}[X_{i+1} \mid \mathcal{F}_i](\omega) \leq X_i(\omega) - \eps\cdot\mathbb{I}_{T>i}(\omega)$, for each $i\geq 0$ and $\omega\in\Omega$.
\end{itemize}
The name comes from the connection with {\em supermartingale processes}~\cite{Williams91}. A supermartingale with respect to a given filtration $\{\mathcal{F}_i\}_{i=0}^{\infty}$ is a stochastic process $(X_i)_{i=0}^{\infty}$ which satisfies conditions~1 and~3 above with $\eps=0$ (note, since $\eps=0$ we define supermartingales with respect to the filtration and not the stopping time). Hence, RSMs are a sublcass of supermartingales.

To conclude this overview of martingale theory, we state two results on RSMs that we will use in our proofs. The first is a result on ranking supermartingales, whose variant was presented in works on termination analysis of probabilistic programs~\cite{FioritiH15,ChatterjeeFNH16}. Since our variant slightly differs from the ones presented in those works, we provide the proof for the sake of completeness (while previous variants consider the special case of stopping times defined by the first time after which the value of an RSM falls below some given threshold, our variant considers general stopping times).

\begin{proposition}\label{prop:rsm}
Let $(\Omega,\mathcal{F},\mathbb{P})$ be a probability space, let $(\mathcal{F}_i)_{i=0}^\infty$ be a filtration and let $T$ be a stopping time with respect to $(\mathcal{F}_i)_{i=0}^\infty$. Suppose that $(X_i)_{i=0}^{\infty}$ is an $\eps$-RSM with respect to $T$, for some $\eps>0$. Then
\begin{enumerate}
    \item $\mathbb{P}[T<\infty] = 1$,
    \item $\mathbb{E}[T] \leq \frac{\mathbb{E}[X_0]}{\eps}$, and
    \item $\mathbb{P}[ T\geq t] \leq \frac{\mathbb{E}[X_0]}{\eps\cdot t}$, for each $t\in\mathbb{N}$.
\end{enumerate}
\end{proposition}

\begin{proof}
We first prove by induction on $i$ that, for each $i\in\mathbb{N}$,
\begin{equation}\label{eq:ineq}
    \mathbb{E}[X_i] \leq \mathbb{E}[X_0] - \eps \cdot \sum_{j=0}^{i-1}\mathbb{P}[T>j].
\end{equation}
The base case $i=1$ follows immediately from the definition of an $\eps$-RSM. We now suppose that the claim is true for $i$, and we prove it for $i+1$:
\begin{equation*}
\begin{split}
    \mathbb{E}[X_{i+1}] &= \mathbb{E}\Big[\mathbb{E}[X_{i+1}\mid \mathcal{F}_i]\Big] \leq \mathbb{E}[X_i] - \eps\cdot\mathbb{P}[T>i] \\
    &\leq \mathbb{E}[X_0] - \eps\cdot\sum_{j=0}^i\mathbb{P}[T>i],
\end{split}
\end{equation*}
where the first inequality holds since $(X_i)_{i=0}^{\infty}$ is an $\eps$-RSM with respect to $T$, and the second inequality holds by the induction hypothesis. Hence, the claim follows.

But we know that $X_i(\omega)\geq 0$ for each $i,\omega$ by the definition of an $\eps$-RSM, hence by plugging $\mathbb{E}[X_i]\geq 0$ into eq.~(\ref{eq:ineq}) we conclude that 
$0 \leq \mathbb{E}[X_0] - \eps \cdot \sum_{j=0}^{i-1}\mathbb{P}[T>j]$ for each $i\in \mathbb{N}$, and thus
\begin{equation*}
    \sum_{j=0}^{\infty}\mathbb{P}[T>j] \leq \frac{\mathbb{E}[X_0]}{\eps} < \infty.
\end{equation*}
It then follows that:
\begin{enumerate}
    \item $\mathbb{P}[T=\infty] = \lim_{t\rightarrow\infty}\mathbb{P}[T>t]=0$, as $\sum_{j=0}^{\infty}\mathbb{P}[T>j]$ converges,
    \item $\mathbb{E}[T] = \sum_{t=0}^{\infty}\mathbb{P}[T>t] \leq \frac{\mathbb{E}[X_0]}{\eps}$, and
    \item $\mathbb{P}[T\geq t]\leq \frac{\mathbb{E}[T]}{t} \leq \frac{\mathbb{E}[X_0]}{\eps\cdot t}$, where the first inequality follows by Markov's inequality.
\end{enumerate}
Hence, the proposition claim follows.
\end{proof}

The second is a classical result on supermartingales (and therefore RSMs) called Azuma's inequality~\cite{azuma1967weighted}, that we will later use in proving the concentration bounds on the stabilization time.

\begin{proposition}[Azuma's inequality]\label{prop:azuma}
Let $(\Omega,\mathcal{F},\mathbb{P})$ be a probability space and let $(\mathcal{F}_i)_{i=0}^\infty$ be a filtration. Suppose that $(X_i)_{i=0}^{\infty}$ is a supermartingale with respect to $(\mathcal{F}_i)_{i=0}^\infty$, and let $(c_i)_{i=0}^\infty$ be a sequence of positive real numbers such that $|X_{i+1}(\omega)-X_i(\omega)|\leq c_i$ for each $i\geq 0$ and $\omega\in\Omega$. Then, for each $n \in \mathbb{N}$ and $t > 0$, we have that
\begin{equation}
    \mathbb{P}\Big[ X_n - X_0 \geq t \Big] \leq e^{\frac{-t^2}{2 \cdot \sum_{i=0}^{n-1} c_i^2}}.
\end{equation}
\end{proposition}

\subsection{Theorem Proofs}\label{sec:thmproofs}

We now proceed to proving Theorem~1 and Theorem~2 from the main text of the paper. Recall that, in the preliminaries, we noted that for each initial state $\mathbf{x}_0\in\mathcal{X}$, the system dynamics induces a Markov process which gives rise to the probability space over the set of all system trajectories that start in the initial state $\mathbf{x}_0$~\cite[Section 2]{Puterman94}. Denote this probability space by $(\Omega_{\mathbf{x}_0},\mathcal{F}_{\mathbf{x}_0},\mathbb{P}_{\mathbf{x}_0})$. The idea behind the proofs of both theorems is to show that any RSM for the stabilizing set $\Stable$ gives rise to a mathematical RSM in the probability space $(\Omega_{\mathbf{x}_0},\mathcal{F}_{\mathbf{x}_0},\mathbb{P}_{\mathbf{x}_0})$. We then use Proposition~\ref{prop:rsm} to deduce a.s.~asymptotic stability in Theorem~1, and Proposition~\ref{prop:rsm} and Proposition~\ref{prop:azuma} to deduce bounds on the stabilization time in Theorem~2. But in order to formally show that an RSM can be instantiated as a mathematical object in this probability space, we first need to define the canonical filtration in this probability space. This will allow defining the stabilization time as a stopping time with respect to the canonical filtration, and finally instantiating the RSM in our stochastic dynamical system as a mathematical RSM with respect to this stopping time.

The rest of this section formalizes the intuition behind the proofs that was outlined above.

\paragraph{Canonical filtration} Fix an initial state $\mathbf{x}_0\in\mathcal{X}$ and consider the probability space $(\Omega_{\mathbf{x}_0},\mathcal{F}_{\mathbf{x}_0},\mathbb{P}_{\mathbf{x}_0})$. For each $i\in\mathbb{N}_0$, define $\mathcal{F}_i\subseteq \mathcal{F}$ to be the $\sigma$-algebra containing the subsets of $\Omega_{\mathbf{x}_0}$ that, intuitively, contain all trajectories in $\Omega_{\mathbf{x}_0}$ whose first $i$ states satisfy some specified property.

Formally, we define $\mathcal{F}_i$ as follows. For each $j\in\mathbb{N}_0$, let $C_j:\Omega_{\mathbf{x}_0}\rightarrow \mathcal{X}$ be a map which to each trajectory $\rho=(\mathbf{x}_t,\mathbf{u}_t,\omega_t)_{t\in\mathbb{N}_0}\in\Omega_{\mathbf{x}_0}$ assigns the $j$-th state $\mathbf{x}_j$ along the trajectory. Then $\mathcal{F}_i$ is the smallest $\sigma$-algebra over $\Omega_{\mathbf{x}_0}$ with respect to which $C_0, C_1, \dots, C_i$ are all measurable, where $\mathcal{X}\subseteq\mathbb{R}^m$ is equipped with the induced Borel-$\sigma$-algebra~\cite[Section 1]{Williams91}. Clearly $\mathcal{F}_0\subseteq\mathcal{F}_1\subseteq\dots$. We say that the sequence of $\sigma$-algebras $(\mathcal{F}_i)_{i=0}^\infty$ is the {\em canonical filtration} in the probability space $(\Omega_{\mathbf{x}_0},\mathcal{F}_{\mathbf{x}_0},\mathbb{P}_{\mathbf{x}_0})$.

\paragraph{Stabilization stopping time} Recall, we use $\Stable\subseteq\mathcal{X}$ to denote the region whose stability we wish to analyze. In order to formally reason about its a.s.~asymptotic stability as well as the stabilization time, we formalize the notion of the stabilization stopping time. Define $T_{\Stable}:\Omega_{\mathbf{x}_0}\rightarrow\mathbb{N}_0\cup\{\infty\}$ to be the first hitting time of the set $\Stable$. Since we assumed that $\Stable$ is closed under system dynamics, for each trajectory $\rho\in\Omega_{\mathbf{x}_0}$, $T_{\Stable}(\rho)$ is exactly the stabilization time of the trajectory. Since whether $T_{\Stable}(\rho)\leq i$ depends solely on the first $i$ states along $\rho$, we clearly have $\{\rho\in\Omega_{\mathbf{x_0}}\mid T(\rho)\leq i\}\in \mathcal{F}_i$ for each $i$ and so $T_{\Stable}$ is a stopping time with respect to $(\mathcal{F}_i)_{i=0}^\infty$. We call $T_{\Stable}$ the {\em stabilization stopping time}.

We now prove the theorems.

\begin{manualtheorem}{1}
Let $f:\mathcal{X}\times\mathcal{U} \times \mathcal{N}\rightarrow \mathcal{X}$ be a Lipschitz continuous dynamics function, $\pi:\mathcal{X}\rightarrow\mathcal{X}$ a Lipschitz continuous policy and $d$ a distribution over $\mathcal{N}$. Suppose that $\mathcal{X}$ is compact and let $\Stable\subseteq\mathcal{X}$ be closed under system dynamics and have a non-empty interior. Suppose that there exists an RSM $V:\mathcal{X}\rightarrow\mathbb{R}$ for $\Stable$. Then $\Stable$ is a.s.~asymptotically stable.
\end{manualtheorem}

\begin{proof}
In order to prove that $\Stable$ is a.s.~asymptotically stable as in Definition~1 we need to show that, for each $\mathbf{x}_0\in\mathcal{X}$,
    \[ \mathbb{P}_{\mathbf{x}_0}\Big[ \lim_{t\rightarrow\infty}d(\mathbf{x}_t,\Stable) = 0 \Big] = 1. \]

Let $\mathbf{x}_0\in\mathcal{X}$. If $\mathbf{x}_0\in\Stable$, then the claim trivially holds since $\Stable$ is closed under system dynamics. Thus suppose without loss of generality that $\mathbf{x}_0\not\in\Stable$, and consider the probability space $(\Omega_{\mathbf{x}_0},\mathcal{F}_{\mathbf{x}_0},\mathbb{P}_{\mathbf{x}_0})$, the canonical filtration $(\mathcal{F}_i)_{i=0}^\infty$ and the stabilziation stopping time $T_{\Stable}$ in it.

Now, we define a stochastic process $(X_i)_{i=0}^\infty$ in $(\Omega_{\mathbf{x}_0},\mathcal{F}_{\mathbf{x}_0},\mathbb{P}_{\mathbf{x}_0})$ via
\begin{equation*}
    X_i(\rho) = \begin{cases}
    V(\mathbf{x}_i), &\text{if } i<T_{\Stable}(\rho) \\
    V(\mathbf{x}_{T_{\Stable}(\rho)}), &\text{otherwise}
    \end{cases}
\end{equation*}
for each $i\geq 0$ and $\rho=(\mathbf{x}_t,\mathbf{u}_t,\omega_t)_{t\in\mathbb{N}_0}\in\Omega_{\mathbf{x}_0}$. Hence, if the stopping time $T_{\Stable}$ is not exceeded by time $i$ we define $X_i$ to be equal to the value of $V$ at the $i$-th state along the trajectory, and after $T_{\Stable}$ is exceeded we define $X_i$ to be equal to the value of $V$ at the time step $T_{\Stable}$ at which the process is stopped.

We claim that $(X_i)_{i=0}^\infty$ is an $\eps$-RSM with respect to the stabilization stopping time $T_{\Stable}$. To prove this claim, we check each of the three defining properties of $\eps$-RSMs:
\begin{itemize}
    \item {\em Each $X_i$ is $\mathcal{F}_i$-measurable.} The value of $X_i$ is defined in terms of the first $i$ states along a trajectory if $T_{\Stable}>i$, and in terms of the first $T_{\Stable}$ states if $i\geq T_{\Stable}$. By the definition of the canonical filtration, we have that $X_i$ is $\mathcal{F}_i$-measurable for each $i\geq 0$.
    \item {\em Each $X_i(\rho)\geq 0$.} Since each $X_i$ is defined in terms of $V$ and since we know that $V(\mathbf{x})\geq 0$ for each state $\mathbf{x}\in\mathcal{X}$, it follows that $X_i(\rho)\geq 0$ for each $i\geq 0$ and $\rho\in\Omega_{\mathbf{x}_0}$.
    \item {\em Each $\mathbb{E}[X_{i+1} \mid \mathcal{F}_i](\rho) \leq X_i(\rho) - \eps\cdot \mathbb{I}_{T_{\Stable}>i}(\rho)$.} First, note that the conditional expectation exists since $X_{i+1}$ is nonnegative for each $i\geq 0$. In order to prove the inequality, we distinguish two cases.
    
    First, if $T_{\Stable}(\rho)>i$, we need to show that  $\mathbb{E}[X_{i+1} \mid \mathcal{F}_i](\rho) \leq X_i(\rho) - \eps$. Let $\rho=(\mathbf{x}_t,\mathbf{u}_t,\omega_t)_{t\in\mathbb{N}_0}$. We have that $X_i(\rho)=V(\mathbf{x}_i)$. On the other hand, we have $\mathbb{E}[X_{i+1} \mid \mathcal{F}_i](\rho) = \mathbb{E}_{\omega\sim d}[V(f(\mathbf{x}_i,\pi(\mathbf{x}_i),\omega)]$. To see this, observe that $\mathbb{E}_{\omega\sim d}[V(f(\mathbf{x}_i,\pi(\mathbf{x}_i),\omega)]$ satisfies all the defining properties of conditional expectation since it is the expected value of $V$ at a subsequent state of $\mathbf{x}_i$, and recall that conditional expectation is a.s.~unique whenever it exists. We thus have
    \begin{equation*}
    \begin{split}
        \mathbb{E}[X_{i+1} \mid \mathcal{F}_i](\rho) &= \mathbb{E}_{\omega\sim d}[V(f(\mathbf{x}_i,\pi(\mathbf{x}_i),\omega)] \\
                                                     &\leq V(\mathbf{x}_i) - \eps = X_i(\rho) - \eps,
    \end{split}
    \end{equation*}
    where the inequality holds since $V$ is an RSM for $\Stable$ and $\mathbf{x}_i\not\in\Stable$ (as $T_{\Stable}(\rho)>i$). This proves the claim.
    
    Second, if $T_{\Stable}(\rho)\leq i$, we need to show that $\mathbb{E}[X_{i+1} \mid \mathcal{F}_i](\rho) \leq X_i(\rho)$. Let $\rho=(\mathbf{x}_t,\mathbf{u}_t,\omega_t)_{t\in\mathbb{N}_0}$. We have $X_i(\rho)=V(\mathbf{x}_{T_{\Stable}(\rho)})$ and $\mathbb{E}[X_{i+1} \mid \mathcal{F}_i](\rho)] = V(\mathbf{x}_{T_{\Stable}(\rho)})$, so the equality follows.
\end{itemize}
This concludes the proof that $(X_i)_{i=0}^\infty$ is an $\eps$-RSM with respect to $T_{\Stable}$.

Therefore, by the first part of Proposition~1, we have that $\mathbb{P}_{\mathbf{x}_0}[T_{\Stable}<\infty]=1$ which by the definition of the stabilization stopping time implies $\mathbb{P}_{\mathbf{x}_0}[ \lim_{t\rightarrow\infty}d(\mathbf{x}_t,\Stable) = 0] = 1$. Since the initial state $\mathbf{x}_0$ was arbitrary, the first point in the definition of a.s.~asymptotic stability is proved.
\end{proof}

\begin{manualtheorem}{2}
Let $f:\mathcal{X}\times\mathcal{U} \times \mathcal{N}\rightarrow \mathcal{X}$ be a Lipschitz continuous dynamics function, $\pi:\mathcal{X}\rightarrow\mathcal{U}$ a Lipschitz continuous policy and $d$ a distribution over $\mathcal{N}$. Suppose that $\mathcal{X}$ is compact and let $\Stable\subseteq\mathcal{X}$ be closed under system dynamics and have a non-empty interior. Suppose that there exists an $\eps$-RSM $V:\mathcal{X}\rightarrow\mathbb{R}$ for $\Stable$. Then, for any initial state $\mathbf{x}_0\in\mathcal{X}$,
\begin{enumerate}
    \item $\mathbb{E}_{\mathbf{x}_0}[ T_{\Stable} ] \leq \frac{V(\mathbf{x}_0)}{\eps}$.
    \item $\mathbb{P}_{\mathbf{x}_0}[ T_{\Stable} \geq t ] \leq \frac{V(\mathbf{x}_0)}{\eps \cdot t}$, for any time $t\in\mathbb{N}$.
    \item If the system has {\em $c$-bounded differences} for $c>0$, then $\mathbb{P}_{\mathbf{x}_0}[ T_{\Stable} \geq t ] \leq A\cdot e^{-t\cdot \eps^2/(2\cdot (c+\eps)^2)}$ for any time $t\in\mathbb{N}$ and $A=e^{\eps\cdot V(\mathbf{x}_0)/(c+\eps)^2}$.
\end{enumerate}
\end{manualtheorem}

\begin{proof}
In the proof of Theorem~1, we showed that $(X_i)_{i=0}^\infty$ is an $\eps$-RSM with respect to $T_{\Stable}$. The first two parts of Theorem~2 then follow from the second and the third part of Proposition~1.

The proof of the third part of Theorem~2 is similar to the argument in~\cite[Section~5.1.2]{ChatterjeeFNH16} which derives concentration bounds on the termination time in probabilistic programs. We define another stochastic process $(Y_i)_{i=0}^\infty$ as follows:
\begin{equation*}
    Y_i(\rho) = X_i(\rho) + \eps\cdot\min\{i,T_{\Stable}(\rho)\}.
\end{equation*}
We claim that $(Y_i)_{i=0}^\infty$ is a supermartingale with respect to the canonical filtration $(\mathcal{F}_i)_{i=0}^\infty$. By applying Azuma's inequality to this newly constructed supermartingale, we will then deduce claim of Theorem~3.

To prove this claim, note that each $Y_i$ is $\mathcal{F}_i$-measurable and nonnegative clearly hold, so we just need to check that $\mathbb{E}[Y_{i+1}\mid\mathcal{F}_i](\rho)\leq Y_i(\rho)$ for each $i\geq 0$ and $\rho\in\Omega_{\mathbf{x}_0}$. To prove this, observe that
\begin{equation}\label{eq:seq}
\begin{split}
    &\mathbb{E}[Y_{i+1}\mid\mathcal{F}_i](\rho) = \mathbb{E}\Big[X_{i+1} + \eps\cdot\min\{i+1,T_{\Stable}\}\mid\mathcal{F}_i\Big](\rho) \\
    &= \mathbb{E}\Big[X_{i+1} \mid\mathcal{F}_i\Big](\rho) + \eps\cdot\mathbb{E}\Big[\min\{i+1,T_{\Stable}\}\mid\mathcal{F}_i\Big](\rho) \\
    &\leq X_i(\rho) - \eps\cdot\mathbb{I}_{T_{\Stable}>i}(\rho) + \eps\cdot \mathbb{E}[\min\{i+1,T_{\Stable}\}\mid\mathcal{F}_i](\rho)
\end{split}
\end{equation}
Now, we distinguish two cases.
\begin{itemize}
    \item If $T_{\Stable}(\rho)>i$, then $\min\{i+1,T_{\Stable}\}=i+1$ and so $\mathbb{E}[\min\{i+1,T_{\Stable}\}\mid\mathcal{F}_i](\rho)=i+1$. Then, continuing on the right-hand-side of eq.~(\ref{eq:seq}), we have
    \begin{equation*}
    \begin{split}
    &\mathbb{E}[Y_{i+1}\mid\mathcal{F}_i](\rho) = \mathbb{E}\Big[X_{i+1} + \eps\cdot\min\{i+1,T_{\Stable}\}\mid\mathcal{F}_i\Big](\rho) \\
    &\leq X_i(\rho) - \eps\cdot\mathbb{I}_{T_{\Stable}>i}(\rho) + \eps\cdot \mathbb{E}[\min\{i+1,T_{\Stable}\}\mid\mathcal{F}_i](\rho) \\
    &= X_i(\rho) - \eps + \eps\cdot (i+1) \\
    &= X_i(\rho) - \eps\cdot i =Y_i(\rho)
    \end{split}
    \end{equation*}
    
    \item If $T_{\Stable}(\rho)\leq i$, then $\min\{i+1,T_{\Stable}\}(\rho)=T_{\Stable}(\rho) = T_{\Stable}(\rho) \cdot \mathbb{I}_{T_{\Stable}\leq i}(\rho)$. But the random variable $T_{\Stable}\cdot \mathbb{I}_{T_{\Stable}\leq i}$ is $\mathcal{F}_i$-measurable, so by the properties of conditional expectation we have that $\mathbb{E}[T_{\Stable}\cdot \mathbb{I}_{T_{\Stable}\leq i}\mid\mathcal{F}_i] = T_{\Stable}\cdot \mathbb{I}_{T_{\Stable}\leq i}$. Plugging this back into the right-hand-side of eq.~(\ref{eq:seq}), we have
    \begin{equation*}
    \begin{split}
    &\mathbb{E}[Y_{i+1}\mid\mathcal{F}_i](\rho) = \mathbb{E}\Big[X_{i+1} + \eps\cdot\min\{i+1,T_{\Stable}\}\mid\mathcal{F}_i\Big](\rho) \\
    &\leq X_i(\rho) - \eps\cdot\mathbb{I}_{T_{\Stable}>i}(\rho) + \eps\cdot \mathbb{E}[\min\{i+1,T_{\Stable}\}\mid\mathcal{F}_i](\rho) \\
    &= X_i(\rho) - 0 + \eps\cdot T_{\Stable}(\rho) = Y_i(\rho).
    \end{split}
    \end{equation*}
\end{itemize}
Hence, $(Y_i)_{i=0}^\infty$ is a supermartingale with respect to the canonical filtration $(\mathcal{F}_i)_{i=0}^\infty$. Moreover, note that $(Y_i)_{i=0}^\infty$ has $(c+\eps)$-bounded differences, as $(X_i)_{i=0}^\infty$ has $c$-bounded differneces.

Finally, the claim of Theorem~3 follows from the following sequence of inequalities
\begin{equation*}
\begin{split}
    &\mathbb{P}_{\mathbf{x}_0}\Big[ T_{\Stable} \geq t \Big] = \mathbb{P}_{\mathbf{x}_0}\Big[ X_t \geq 0 \land T_{\Stable} \geq t \Big] \\
    &=\mathbb{P}_{\mathbf{x}_0}\Big[ X_t + \eps\cdot\min\{t, T_{\Stable}\} - X_0 \geq \eps\cdot\min\{t, T_{\Stable}\} - X_0 \\
    &\hspace{2.5cm} \land T_{\Stable} \geq t \Big] \\
    &=\mathbb{P}_{\mathbf{x}_0}\Big[ X_t + \eps\cdot\min\{t, T_{\Stable}\} - X_0 \geq \eps\cdot t - X_0 \land T_{\Stable} \geq t \Big] \\
    &\leq \mathbb{P}_{\mathbf{x}_0}\Big[ X_t + \eps\cdot\min\{t, T_{\Stable}\} - X_0 \geq \eps\cdot t - X_0 \Big] \\
    &= \mathbb{P}_{\mathbf{x}_0}\Big[ Y_t - Y_0 \geq \eps\cdot t - X_0 \Big] \\
    &\leq e^{\frac{-(\eps \cdot t-X_0)^2}{2 \cdot \sum_{i=0}^{t-1} (c+\eps)^2}} = e^{\frac{-(\eps \cdot t-X_0)^2}{2 \cdot t\cdot (c+\eps)^2}} \\
    &= e^{\frac{-\eps^2 \cdot t}{2 \cdot (c+\eps)^2}} \cdot e^{\frac{\eps\cdot X_0}{(c+\eps)^2}} \cdot e^{\frac{-X_0^2}{2 \cdot t \cdot (c+\eps)^2}} \\
    &= e^{\frac{-\eps^2 \cdot t}{2 \cdot (c+\eps)^2}} \cdot e^{\frac{\eps \cdot V(\mathbf{x}_0)}{(c+\eps)^2}} \cdot e^{\frac{-V(\mathbf{x}_0)^2}{2 \cdot t \cdot (c+\eps)^2}} \\
    &\leq e^{\frac{-\eps^2 \cdot t}{2 \cdot (c+\eps)^2}} \cdot e^{\frac{\eps\cdot V(\mathbf{x}_0)}{(c+\eps)^2}} \cdot 1 \\
    &= A \cdot e^{-t\cdot \frac{\eps^2}{2 \cdot (c+\eps)^2}}
\end{split}
\end{equation*}
with $A=e^{\eps \cdot V(\mathbf{x}_0)/(c+\eps)^2}$, where in sixth row we applied Azuma's inequality to the supermartingale $(Y_i)_{i=0}^\infty$ and in the ninth row we use $e^{-V(\mathbf{x}_0)^2/(2 \cdot t \cdot (c+\eps)^2)}\leq 1$.
\end{proof}

\section{Proofs of Theorem~3 and Theorem~4}

\begin{manualtheorem}{3}
Suppose that the verifier in Algorithm~1 verifies that $V$ satisfies eq.~(3) for each $\mathbf{x}\in\tilde{\mathcal{X}}$. Let $-m\in\mathbb{R}$ be such that $V(\mathbf{x})\geq -m$ for each $\mathbf{x}\in\mathcal{X}$. Then, the function $V'(\mathbf{x})=V(\mathbf{x})+m$ is an RSM for $\Stable$. Hence, $\Stable$ is a.s.~asymptotically stable.
\end{manualtheorem}

\begin{proof}
The function $V'$ is continuous since $V$ is continuous. Moreover, $V'$ is nonnegative since $V'(\mathbf{x}) = V(\mathbf{x})+m \geq -m+m = 0$ for each $\mathbf{x}\in\mathcal{X}$. Thus, we only need to show that $V'$ satisfies the expected decrease condition for each $\mathbf{x}\in\mathcal{X}\backslash\Stable$, i.e.~that there exists $\eps>0$ such that
\[ \mathbb{E}_{\omega\sim d}\Big[ V \Big( f(\mathbf{x}, \pi(\mathbf{x}), \omega) \Big) \Big] \leq V(\mathbf{x}) - \eps \]
for each $\mathbf{x}\in\mathcal{X}\backslash\Stable$.

We prove that $\eps>0$ defined via
\[ \eps = \min_{\mathbf{x}\in \tilde{\mathcal{X}}} \Big( V(\mathbf{x}) - \tau \cdot K - \mathbb{E}_{\omega\sim d}\Big[ V \Big( f(\mathbf{x}, \pi(\mathbf{x}), \omega) \Big) \Big] \Big) \]
satisfies the claim, where by theorem assumptions we know that $\eps$ is indeed strictly positive.

To show this, fix $\mathbf{x}\in\mathcal{X}\backslash\Stable$ and let $\tilde{\mathbf{x}}\in\tilde{\mathcal{X}}$ be such that $||\mathbf{x}-\tilde{\mathbf{x}}||_1 \leq \tau$. Such $\tilde{\mathbf{x}}$ exists by the definition of a discretization. Then, by Lipschitz continuity of $f$, $\pi$ and $V$, we have
\begin{equation}\label{eq:long}
\begin{split}
    &\mathbb{E}_{\omega\sim d}\Big[ V \Big( f(\mathbf{x}, \pi(\mathbf{x}), \omega) \Big) \Big] \\
    &\leq \mathbb{E}_{\omega\sim d}\Big[ V \Big( f(\tilde{\mathbf{x}}, \pi(\tilde{\mathbf{x}}), \omega) \Big) \Big] \\
    &\hspace{1cm}+ ||f(\tilde{\mathbf{x}}, \pi(\tilde{\mathbf{x}}), \omega) - f(\mathbf{x}, \pi(\mathbf{x}), \omega)||_1 \cdot L_V \\
    &\leq \mathbb{E}_{\omega\sim d}\Big[ V \Big( f(\tilde{\mathbf{x}}, \pi(\tilde{\mathbf{x}}), \omega) \Big) \Big] \\
    &\hspace{1cm}+ ||(\tilde{\mathbf{x}}, \pi(\tilde{\mathbf{x}}), \omega) - (\mathbf{x}, \pi(\mathbf{x}), \omega)||_1 \cdot L_V\cdot L_f \\
    &\leq \mathbb{E}_{\omega\sim d}\Big[ V \Big( f(\tilde{\mathbf{x}}, \pi(\tilde{\mathbf{x}}), \omega) \Big) \Big] \\
    &\hspace{1cm}+ ||\tilde{\mathbf{x}} - \mathbf{x}||_1 \cdot L_V\cdot L_f \cdot (1 + L_{\pi}) \\
    &\leq \mathbb{E}_{\omega\sim d}\Big[ V \Big( f(\tilde{\mathbf{x}}, \pi(\tilde{\mathbf{x}}), \omega) \Big) \Big] \\
    &\hspace{1cm}+ \tau \cdot L_V\cdot L_f \cdot (1 + L_{\pi}),
\end{split}
\end{equation}
where in the last row we used $||\mathbf{x}-\tilde{\mathbf{x}}||_1 < \tau$.
On the other hand, by Lipschitz continuity of $V$ we have
\begin{equation}\label{eq:short}
    V(\mathbf{x}) \geq V(\tilde{\mathbf{x}}) + ||\tilde{\mathbf{x}} - \mathbf{x}||_1 \cdot L_V \geq V(\tilde{\mathbf{x}}) - \tau\cdot L_V.
\end{equation}
Thus combining eq.(\ref{eq:long}) and (\ref{eq:short}) we get
\begin{equation}
\begin{split}
    &V(\mathbf{x}) - \mathbb{E}_{\omega\sim d}\Big[ V \Big( f(\mathbf{x}, \pi(\mathbf{x}), \omega) \Big) \Big] \\
    &\geq V(\tilde{\mathbf{x}}) - \tau\cdot L_V - \mathbb{E}_{\omega\sim d}\Big[ V \Big( f(\tilde{\mathbf{x}}, \pi(\tilde{\mathbf{x}}), \omega) \Big) \Big] \\
    &\hspace{1cm}- \tau \cdot L_V\cdot L_f \cdot (1 + L_{\pi}) \\
    &= V(\tilde{\mathbf{x}}) - \tau\cdot K - \mathbb{E}_{\omega\sim d}\Big[ V \Big( f(\tilde{\mathbf{x}}, \pi(\tilde{\mathbf{x}}), \omega) \Big) \Big] \\
    &\geq \eps
\end{split}
\end{equation}
where the equality in the second last row follows by the definition of $K$, and the inequality in the last row follows by our choice of $\eps$. This concludes the proof that $V'$ is an RSM for $\Stable$. The claim that $\Stable$ is a.s.~asymptotically stable then follows from Theorem~1.

\end{proof}

\begin{manualtheorem}{4}
Let $M = \min_{\mathbf{x}\in\tilde{\mathcal{X}}}|\mathcal{D}_{\mathbf{x}}|$. If $V_{\theta}$ satisfies eq.~(3) for each $\mathbf{x}\in\tilde{\mathcal{X}}$ and if $L_{V_{\theta}} \leq \delta / (\tau\cdot (L_f \cdot (L_\pi + 1) + 1))$, then $\lim_{M\rightarrow \infty} \mathcal{L}(\theta) = 0$ holds almost-surely.
\end{manualtheorem}

\begin{proof}
Since $L_{V_{\theta}} \leq \delta / (\tau\cdot (L_f \cdot (L_\pi + 1) + 1))$, we have that $\mathcal{L}_{\text{Lipschitz}}(\theta)=0$. Thus,
\begin{equation*}
\begin{split}
    &\mathcal{L}(\theta) = \mathcal{L}_{\text{RSM}}(\theta) \\
    &= \frac{1}{|\tilde{\mathcal{X}}|}\sum_{\mathbf{x}\in\tilde{\mathcal{X}}}\Big( \max\Big\{\sum_{\mathbf{x}'\in\mathcal{D}_{\mathbf{x}}}\frac{V_{\theta}(\mathbf{x}')}{|\mathcal{D}_{\mathbf{x}}|}-V_{\theta}(\mathbf{x}) + \tau \cdot K, 0\Big\} \Big).
\end{split}
\end{equation*}
Hence, it suffices to prove that for each $\mathbf{x}\in\tilde{\mathcal{X}}$ we almost-surely have
\[ \lim_{M\rightarrow\infty}\max\Big\{\sum_{\mathbf{x}'\in\mathcal{D}_{\mathbf{x}}}\frac{V_{\theta}(\mathbf{x}')}{|\mathcal{D}_{\mathbf{x}}|}-V_{\theta}(\mathbf{x}) + \tau \cdot K, 0\Big\} = 0. \]

To prove this, observe that $\{V_{\theta}(\mathbf{x}')\mid \mathbf{x}'\in \mathcal{D}_{\mathbf{x}} \}$ is a set of values of $V$ in at least $M$ independently sampled successor states of $\mathbf{x}$ according to the distribution over the successor states of $\mathbf{x}$ defined by the system dynamics and the probability distribution $d$ over disturbance vectors. Since $\mathcal{X}$ is compact and $V_{\theta}$ is continuous, the random value defined by the value of $V$ at a randomly sampled successor state of $\mathbf{x}$ is bounded, thus has a well-defined and finite first moment.

The Strong Law of Large Numbers~\cite[Section 12.10]{Williams91} states that, given a distribution $\mu$ with a finite first moment and a sequence $X_1,X_2,\dots$ of independent identically distributed random variables distributed according to $\mu$, we have that
\[ \lim_{n\rightarrow\infty}\frac{X_1+\dots+X_n}{n} = \mathbb{E}_{X\sim\mu}[X] \]
holds almost-surely.

Applying the Strong Law of Large Numbers to $\{V_{\theta}(\mathbf{x}')\mid \mathbf{x}'\in \mathcal{D}_{\mathbf{x}} \}$ we conclude that, almost-surely,
\begin{equation*}
\begin{split}
    &\lim_{M\rightarrow\infty}\max\Big\{\sum_{\mathbf{x}'\in\mathcal{D}_{\mathbf{x}}}\frac{V_{\theta}(\mathbf{x}')}{|\mathcal{D}_{\mathbf{x}}|}-V_{\theta}(\mathbf{x}) + \tau \cdot K, 0\Big\}  \\
    &= \max\Big\{\lim_{M\rightarrow\infty}\sum_{\mathbf{x}'\in\mathcal{D}_{\mathbf{x}}}\frac{V_{\theta}(\mathbf{x}')}{|\mathcal{D}_{\mathbf{x}}|}-V_{\theta}(\mathbf{x}) + \tau \cdot K, 0\Big\} \\
    &= \max\Big\{\mathbb{E}_{\omega\sim d}\Big[ V \Big( f(\mathbf{x}, \pi(\mathbf{x}), \omega) \Big) \Big]-V_{\theta}(\mathbf{x}) + \tau \cdot K, 0\Big\} \\
    &= 0,
\end{split}
\end{equation*}
where the first equality holds since limits may be interchanged with the maximum function over finitely many arguments, the second equality holds almost-surely by the Strong Law of Large Numbers, and the third inequality holds since $V_{\theta}$ satisfies eq.~(3) for each $\mathbf{x}\in\tilde{\mathcal{X}}$.

Hence, $\mathcal{L}_{\text{RSM}}(\theta)=0$, as claimed.

\end{proof}

\section{Experimental evaluation details}
We implemented our algorithm in JAX \cite{jax2018github}. All experiments where run on a 4 CPU-core machine with 32GB of memory and an Nvidia GTX 1080 Ti with 11GB of memory. The full source code to reproduce all our experiments is provided in the Supplementary Material.

\subsection{Benchmark environments}
The dynamics of our first task (2D system) is defined as
\begin{equation}
     \mathbf{x}_{t+1} = \begin{pmatrix} 1 & 0.045 \\ 0 & 0.9 \end{pmatrix} \mathbf{x}_t + \begin{pmatrix} 0.45 \\0.5 \end{pmatrix} g(\mathbf{u}_t) + \begin{pmatrix}0.015 & 0 \\0 & 0.005 \end{pmatrix}  \omega, 
\end{equation}
where $\omega$ is a disturbance vector and $\omega[1],\omega[2] \sim \text{Triangular}$. Here, we use square brackets to denote the coordinate index, e.g.~$\omega[1]$ denotes the first coordinate of the two-dimensional disturbance vector $\omega\in\mathbb{R}^2$. The probability density function of $\text{Triangular}$ is defined by 
\begin{equation}
    \text{Triangular}(x) := \begin{cases} 0 & \text{if } x< -1\\ 1 - |x| & \text{if } -1 \leq x \leq 1\\ 0 & \text{otherwise}\end{cases}.
\end{equation}
The function $g$ bounds the range of admissible actions by $g(u) = \max(\min(u,1),-1)$.
For training a policy on the 2D system, we used a reward $r_t$ at time $t$ defined by $r_t := 1 - \mathbf{x}_t[1]^2 - \mathbf{x}_t[2]^2$.

The dynamics function of the inverted pendulum task is defined as 
\begin{align*}
    \mathbf{x}_{t+1}[2] &:= (1-b)  \mathbf{x}_{t}[2] \\
    &+ \Delta \cdot \big( \frac{-1.5 \cdot G \cdot \text{sin}(\mathbf{x}_{t}[1]+\pi)}{2l} + \frac{3}{m l^2} 2g(\mathbf{u}_t)\big)\\
    &+ 0.002 \omega[1]\\
    \mathbf{x}_{t+1}[1] &:=  \mathbf{x}_{t}[1] + \Delta \cdot \mathbf{x}_{t+1}[2] + 0.005 \omega[2],
\end{align*}
where the parameters $\delta, G, m, l, b$ are defined in Table \ref{tab:invpend}.
For training a policy on the inverted pendulum task, we used a reward $r_t$ at time $t$ defined by $r_t := 1 - \mathbf{x}_{t}[1]^2 - 0.1 \mathbf{x}_{t}[2]^2$.

\begin{table}[]
    \centering
    \begin{tabular}{c|c}\toprule
        Parameter & Value  \\\midrule
        $\Delta$ &  0.05\\
        $G$ & 10\\
        $m$ & 0.15\\
        $l$ & 0.5\\
        $b$ & 0.1\\\bottomrule
    \end{tabular}
    \caption{Parameters of the inverted pendulum task.}
    \label{tab:invpend}
\end{table}

The hyperparameters used in our experiments for learning the RSM neural network are listed in Table \ref{tab:algorithm}.

\begin{table}[]
    \centering
    \begin{tabular}{c|c}\toprule
        Parameter & Value \\\midrule
        Optimizer & Adam\cite{kingma2014adam} \\
        Learning rate & 0.0001\\
        $\lambda$ & 0.0005 \\
        $\tau$ & 0.1\\
        $\delta$ &  4.0\\
        $N$ & 20 \\\bottomrule
    \end{tabular}
    \caption{Hyperparameters used for learning the RSM candidate, i.e., the parameters to trade-off between optimizing the decrease objective and the Lipschitz bounds. Note that a larger $\tau$ is used by the learner than the verifier, to account for numerical stability and the overapprimixation introduced by the expected value computation.}
    \label{tab:algorithm}
\end{table}

\subsection{Grid refinement}
We implemented a grid refinement procedure to refine the mesh of the discretization used by the verifier.
The refinement procedure consists of two parts, a scheduled one and an on-demand refinement.
The scheduled refinement multiplies $\tau$ by $0.2$ after 4 unsuccessful learner-verifier iterations.
The on-demand refinement is activated if the decrease condition (eq.~(3) in the main paper) is violated in at least one state in the discretization, but the weaker decrease condition 
\begin{equation}\label{eq:refine}
    \mathbb{E}_{\omega\sim d}\Big[ V \Big( f(\mathbf{x}, \pi(\mathbf{x}), \omega) \Big) \Big] < V(\mathbf{x}),
\end{equation}
i.e.~eq.~(3) without the term $-\tau \cdot K$, is satisfied in all states in the discretization.
Once activated, the verifier decomposes all grid cells defined by the discretization that violate the decrease condition into smaller grid cells with a mesh of $0.1\tau$.
The verifier then starts to check all decomposed cells if they still violate the decrease condition. If such a refined cell violating the decrease condition is found, the procedure is canceled. 
If all refined cell fulfil the decrease condition, the verifier concludes that the learned neural network is an RSM.

We performed an ablation study to test whether our on-demand refinement can speed-up our algorithm, i.e., whether the additional runtime cost of the on-demand refinement is justified by requiring fewer loop iterations. The results in Table \ref{tab:stats2} shows that the effectiveness of our on-demand refinement procedure depends on the particular system and no general statement about whether it helps can be made. 

\begin{table*}
    \centering
    \begin{tabular}{c|cccc}\toprule
    Environment & On-demand refinement & Iterations & Mesh ($\tau$) & Runtime \\\midrule
        \multirow{2}{*}{2D system} & No & 4 & 0.002 & 559 \\
         & Yes & 4 & 0.01 & 3007 \\\hline
        \multirow{2}{*}{Inverted pendulum}& No & 5 & 0.002 & 2020 \\
          & Yes & 2 & 0.01 & 176 \\\bottomrule
    \end{tabular}
    \caption{Number of learner-verifier loop iterations, mesh of the discretization used by the verifier, and total algorithm runtime (in seconds) with and without on-demand cell refinement.}
    \label{tab:stats2}
\end{table*}

\end{document}